\documentclass[runningheads]{llncs}
\usepackage[T1]{fontenc}
\usepackage{graphicx}
\usepackage{amsmath,amssymb}
\usepackage{thmtools}
\usepackage{hyperref}

\makeatletter
\def\Bign#1{\mathclose{\hbox{$\left#1\vbox to9\p@{}\right.\n@space$}}\mathopen{}}
\makeatother

\newcommand\ExOp[1]{\ensuremath{%
  \!\sideset{}{_#1}{\mathop{\scalebox{1}[1.3]{/}}}}}

    \usepackage{mathrsfs}
    \usepackage{multirow}
    \usepackage{makecell}

  \newcommand{\abs}[1]{\vert #1 \vert}
  \newcommand{\angles}[1]{\langle #1 \rangle}

  \usepackage{bbm}

  \newcommand*\bell{\ensuremath{\boldsymbol\ell}}
  \newcommand*\bu{\ensuremath{\boldsymbol{u}}}
  \newcommand*\bm{\ensuremath{\boldsymbol{m}}}
  \newcommand*\bn{\ensuremath{\boldsymbol{n}}}

  \newcommand{\infball}{\mathcal{B}^{\infty}}

  \newcommand{\soundoracle}{\mathtt{sound\_verif}}

  \newcommand{\universe}{\mathbb{F}}
  
  \newcommand{\uu}{\underline{\mathbf{U}}}
  \newcommand{\ou}{\overline{\mathbf{U}}}

    \newcommand{\intervals}[1]{\mathbb{I}(#1)}
    
    \newcommand{\intervalsd}{\mathbb{I}(d)}
    \newcommand{\intervalsdx}{\mathbb{I}(d)\vert_{\mathbf{x}}}
    \newcommand{\intervalsdxu}{\mathbb{I}(d)\vert^{\universe}_{\mathbf{x}}}

  \newcommand{\vol}{v}

  \newcommand{\metric}{\omega}

  \newcommand{\din}{{d_{\text{in}}}}
  \newcommand{\dout}{d_{\text{out}}}

  \newcommand{\uradi}{\underline{r}_i}
  \newcommand{\oradi}{\overline{r}_i}
  \newcommand{\uRadi}{\underline{R}_i}
  \newcommand{\oRadi}{\overline{R}_i}
  \newcommand{\uradj} {\underline{r}}

  \usepackage{pifont}
  \newcommand{\xmark}{\ding{55}}%

  \usepackage[ruled, vlined, linesnumbered]{algorithm2e}
  \SetKwInput{KwInput}{Input}
  \SetKwInput{KwOutput}{Output} 

  \usepackage{adjustbox}
  
  \usepackage{booktabs}

  \usepackage{multicol}

\begin{document}
\title{Interval Certifications for Multilayered Perceptrons via Lattice Traversal}
\titlerunning{Interval Certifications for MLPs via Lattice Traversal}

%
\author{Merkouris Papamichail\inst{1,2} \and
Konstantinos Varsos\inst{1,2}\orcidID{0000-0002-0833-6890} \and
Giorgos Flouris\inst{1}\orcidID{0000-0002-8937-4118} \and
Jo\~ao Marques-Silva\inst{3,4}\orcidID{0000-0002-6632-3086}}
\authorrunning{M. Papamichail et al.}
%
\institute{
  Foundation for Reasearch and Technology - Hellas, Heraklion, Greece
  \email{\{mercoyris,varsosk,fgeo\}@ics.forth.gr} \and
  University of Crete, Heraklion, Greece \and
  Catalan Institution for Research and Advanced Studies, Barcelona, Spain
  \email{jpms@icrea.cat}\and
  University of Lleida, Lleida, Spain
}
\maketitle              
    \begin{abstract}
      In this work we present a rigorous theoretical framework to a foundational problem of AI safety, namely \emph{adversarial robustness}. In particular, we show that the adversarial robustness problem can be reduced to a \emph{lattice traversal} problem. Each element of this lattice corresponds to an \emph{interval}, i.e., an axis-aligned hyper-rectangle, containing an input point $\mathbf{x}$. Consider a \emph{multilayered perceptron classifier (MLP)}. An interval $I$ constitutes a sound certification if $\mathbf{x} \in I$ and $\mathbf{x}$ can be freely perturbed in $I$ without changing the MLP's prediction. Complementarily, an interval $I$ constitutes a complete certification if $\mathbf{x} \in I$ and when $\mathbf{x}$ moves outside of $I$ the MLP's prediction is guaranteed to change. While the sound certification problem corresponds to the well-studied adversarial robustness, complete certifications have not been examined in the literature. We develop lattice traversal operators, which we apply in a refine \& verify iterative scheme. Using formal MLP verifiers, sound maximality and complete minimality are guaranteed. Moreover, we examine objective optimization problems. There we discover some interesting asymmetries. For complete certifications, the minimum solution is obtained in polynomial oracle calls. This does not hold for sound certifications, where we prove strong \emph{intractability} results. Additionally, we examine optimization problems in \emph{symmetric} intervals (i.e., $\ell_\infty$-spheres), where we provide logarithmic algorithms. Finally, we present an empirical evaluation, using the novel \textbf{\texttt{ParallelepipedoNN}}\footnote{\url{https://github.com/merkouris148/parallelepipedonn}} system.\\

    \end{abstract}

    \section{Introduction}
    \label{sec:introduction}
    Artificial Intelligence, mostly driven by deep neural networks (NN), is rapidly becoming part of our everyday life, from recommendation systems in media platforms \cite{netflix} to large language models chat-bots \cite{deepseek}. Despite these achievements, NNs promise even greater accomplishments by replacing humans in critical, decision-making areas, from driving \cite{driving} to healthcare \cite{healthcare} or government administration \cite{gov}. However, NNs are brittle, meaning that small, often imperceptible input perturbations can flip their predictions. These inputs are commonly referred to as \emph{adversarial examples} \cite{goodfellow-adversarial-attacks,meng,szegedy-adversarial-attacks}. 

Ensuring a NN's \emph{robustness} to adversarial attacks remains a persisting problem for AI safety. The first works on the field focused on exploiting the gradient information to produce adversarial examples, subsequently incorporating them into the learning process, e.g., \cite{goodfellow-adversarial-attacks,pgd-adversarial-attack}. Nevertheless, these initial attempts fail to solve the problem in its generality \cite{zhang-2020}. More sophisticated methods utilized the convex relaxation of a NN \cite{ehlers}, reducing adversarial robustness to a convex optimization problem. This problem was either solved directly \cite{kabahala,li,liu}, or in its dual form \cite{wong}. Even so, this family of works suffers from low precision, since they rely on a relaxation of the original problem \cite{mirman-2022,salman-2019}.

The hardness of adversarial robustness stems from the NN representation and seems to be deeply rooted in its computational properties. NN's activation functions introduce nonlinearities that can only be studied using integer constraints. Therefore, NN can only be accurately described as \emph{mixed integer linear programs (MILPs)} \cite{katz}. The MILP description of a NN made it possible to construct sound and complete NN verifiers (e.g., \textsc{Marabou} \cite{marabou,marabou-2}), improving upon earlier \emph{Satisfiability Modulo Theory (SMT)} techniques, e.g., \textsc{Reluplex}~\cite{reluplex}. Formal NN Verifiers prove if $\sigma(X) = Y$, for a pair of I/O-sets $X, Y$, and a given NN $\sigma(\cdot)$. If the property does not hold, they provide a \emph{counterexample}, namely some $\mathbf{x} \in X$, s.t. $\sigma(\mathbf{x}) \notin Y$. However, this precision comes at a cost. Verifying a property on NN is \textbf{NP}-hard \cite{reluplex}. Moreover, verifying that a given area is free of adversarial examples does not admit an approximate algorithm \cite{fast-lin}.

This work aspires to provide a detailed formal analysis on robustness certification. Our work differs from previous attempts \cite{kabahala,li,liu,wong} in considering the underlying problem in its generality. We consider the family of \emph{interval} certifications, i.e., axis-aligned hyper-rectangles, containing a given input $\mathbf{x}$. To our knowledge, this is the most general family of certification considered in the literature \cite{kabahala,li}. Utilizing Sunaga's Interval Algebra \cite{sunaga}, we show that the space of interval certifications is organized as an \emph{innumerable, complete lattice}. We introduce a set of  \emph{lattice traversal operators} that enable systematic exploration of this space. These operators are then applied to refine-and-verify iterative schemes for computing \emph{maximally sound} and \emph{minimally complete} interval certifications. An interval $I$ constitutes \emph{sound certification}, if $\mathbf{x} \in I$ and $\mathbf{x}$ can vary arbitrarily within $I$ without changing model's prediction. Dually, $I$ is a \emph{complete certification}, if $\mathbf{x} \in I$ and any movement of $\mathbf{x}$ outside $I$ is guaranteed to change the prediction. Minimality and maximality are defined, w.r.t. set inclusion. Existing approaches \cite{kabahala,li,liu,wong} compute sound certifications, \emph{without} guaranteeing maximality. Moreover, complete certifications have not, to the best of our knowledge, been considered in the literature. 

Further we examine optimization problems over interval certifications, focusing on the \emph{minimum edge length} objective, which is widely used in prior work \cite{li,liu,wong}. In contrast to existing methods, our approach provides \emph{non-triviality guarantees} for our certifications. Namely, our algorithms can \emph{decide} if a non-trivial solution exists to a given optimization problem, under certain assumptions. This core functionality is lacking in existing methods, due to their reliance on relaxation. We provide a qualitative comparison between existing work and ours in Tbl. \ref{tab:compare}. Finally, we strengthen known intractability results, showing that computing optimal sound interval certifications cannot be achieved in polynomial time, w.r.t. the input dimension, the number of verification calls, and the time of each verification call. 
\begin{table}
    \centering
    \begin{adjustbox}{width=\textwidth}
    \begin{tabular}{l c c c c c c c}
        \toprule\toprule
         Work
         ~~& \thead{MLP\\Repr.} ~~& Cert. ~~& Sound ~~& Comp. ~~& \thead{Max/\\Min} ~~& Obj. ~~& \thead{Non-Triv.}  \\
         \midrule\midrule
         Wong et al. \cite{wong} & Dual & Unif.  & \checkmark & \xmark & \xmark & $\alpha/\mathcal{A}$ & \xmark\\
         Liu et al. \cite{liu} & Dual & Sym. & \checkmark & \xmark & \xmark & $\alpha$ & \xmark\\
         Li et al. \cite{li} & Conv. & Gen. & \checkmark & \xmark & \xmark & $\alpha$ & \xmark\\
         Kabahala et al. \cite{kabahala} & MILP+Conv. & Gen. & \checkmark & \xmark & \checkmark & $\mathcal{A}$ & \xmark\\
         This work & MILP & Gen & \checkmark & \checkmark & \checkmark & $\alpha$ & \checkmark\\
         \bottomrule\bottomrule
    \end{tabular}
    \end{adjustbox}
    \caption{Comparing existing work with \texttt{ParallelepipedoNN}. Conv.: the primary convex approximation. Dual: the dual convex approximation. Unif.: uniform intervals, i.e. $\ell_\infty$-circles. Sym.: symmetric intervals of the form $[\mathbf{x} - \mathbf{e}, \mathbf{x} + \mathbf{e}]$. Gen.: General intervals. Finally, with $\alpha$, we denote the minimum edge length, while with $\mathcal{A}$ the maximum edge length, or \emph{diameter}. \cite{kabahala} offers maximal solutions, but for diameter optimization. \cite{wong} examines uniform intervals, thus the min. edge length $\alpha$ and the diameter $\mathcal{A}$ coincide, for the $\ell_\infty$-norm.}
    \label{tab:compare}
\end{table}

\noindent\textbf{Outline.} In Sec. \ref{sec:preliminaries}, we introduce the necessary preliminaries. Sec. \ref{sec:intervals} develops the interval algebra required for our analysis, while Sec. \ref{sec:approximations} explores the structure of the space of interval certifications. 
In Sec. \ref{sec:optimization} we review optimization problems. Finally, Sec. \ref{sec:implementation} presents the \texttt{ParallelepipedoNN} system and discusses the practical implications of our work.

    \section{Preliminaries}
    \label{sec:preliminaries}
    In this section, we review some elementary notions and definitions that will be needed in the rest of this work. 
For any natural number $d \in \mathbb{N}$, we denote with $[d]$ the set $\{1, 2, \dots, d\}$. Vectors will be denoted by bold, e.g., $\mathbf{x}$, while scalar values by light, e.g., $x$. For a $d$-dimensional vector $\mathbf{x}$ we denote with $x_i$ its $i$-th coordinate. Moreover, let $f\colon \mathbb{R} \to \mathbb{R}$ be a real function that takes as input some $x \in \mathbb{R}$. For a vector input $\mathbf{x} \in \mathbb{R}^d$, we denote with $\mathbf{f}(\mathbf{x}) \in \mathbb{R}^d$ the vector function \emph{induced by $f$}, i.e., $\mathbf{f}(\mathbf{x}) = (f(x_1), \dots, f(x_d))$. The $d$-dimensional vectors $\mathbf{0}$ and $\mathbf{1}$ denote the zero  and the all-ones vectors. For each $i \in [d]$, we denote with $\mathbf{e}^i$ a vector, s.t. $e^i_j = 0$, when $i \neq j$, and $e^i_i = 1$. Finally, matrices $A \in \mathbb{R}^{d_1 \times d_2}$ are denoted with capital letters.

\paragraph{Normed Vector Spaces.}
In this work, we are interested in $d$-dimensional \emph{normed vector spaces} on the \emph{real field} $\mathbb{R}$. We will use $\ell_p$-norms, denoted $\|\mathbf{x}\|_p$ and defined as $\|\mathbf{x}\|_p = \sqrt[p]{\sum_{i \in [d]} \abs{x_i}^p}$. In the limit $p \to \infty$ this reduces to the infinity norm $\|\mathbf{x}\|_\infty = \max_{i \in [d]} \abs{x_i}$.

With $\mathcal{B}^p(\mathbf{x}^\star, \rho) \subset \mathbb{R}^d, \rho > 0$ we denote the $d$-dimensional sphere around $\mathbf{x}$ with radius $\rho$ w.r.t.\ the $\|\cdot\|_p$ norm, i.e., $\mathcal{B}^p(\mathbf{x}^\star, \rho) = \{\mathbf{x} \in \mathbb{R}^d \mid \|\mathbf{x} - \mathbf{x}^\star\|_p \leq \rho\}$. Consider a set $S \subseteq \mathbb{R}^d$. With $\partial S$ we denote the \emph{boundary} of $S$ w.r.t.\ the measure $\|\cdot\|_p$, i.e., $\partial S = \{\mathbf{x} \in S \mid \forall \rho > 0,~~ \mathcal{B}^p(\mathbf{x}, \rho) \cap (\mathbb{R}^d \setminus S) \neq \emptyset\}$. The \emph{interior} of $S$, denoted by $S^\circ$, is composed of the points of $S$ not belonging to the boundary, i.e., $S^\circ = S \setminus \partial S$. A set $S \subseteq \mathbb{R}^d$ is \emph{open} w.r.t.\ the measure $\|\cdot\|_p$ if for every $\mathbf{x} \in S$, there is a $\rho > 0$ such that $\mathcal{B}^p(\mathbf{x}, \rho) \subseteq S$. A set $S \subseteq \mathbb{R}^d$ is \emph{closed} if ~$\mathbb{R}^d \setminus S$ is open. A set $S \subseteq \mathbb{R}^d$ is called \emph{bounded} if there is some finite $\rho > 0$ s.t.\ $\infball(\mathbf{0}, \rho) \supseteq S$. Finally, a closed and bounded set is \emph{compact}.

\paragraph{Multilayered Perceptrons.}
A MLP is a function $\sigma\colon \mathbb{F} \to \mathbb{S}$, with $\mathbb{F} \subset \mathbb{R}^{\din}$, $\mathbb{S} \subset \mathbb{R}^{\dout}$ denoting the features (input) and scores (output) spaces, respectively. We focus on MLPs, with \emph{rectified linear units (ReLU)} ~\cite{relu} as activation functions. For $x \in \mathbb{R}$ the ReLU function $r(x)$ is given as $r(x) = \max(0, x)$. In higher dimensions, we have $\mathbf{r}(\mathbf{x}) = (r(x_1), \dots, r(x_d))$. We give the following formal definition.

\begin{definition}[Multilayered Perceptron]
    \label{def:mlp}
    A \emph{multilayered perceptron} $\sigma \colon \mathbb{F} \to \mathbb{S}$, with $\mathbb{F} \subset \mathbb{R}^{\din}, \mathbb{S} \subset \mathbb{R}^{\dout}$. is described as the tuple $\sigma = \angles{L, D, W, Q}$. With $L \in \mathbb{N}$, we denote the \emph{number of layers}. With $D$, we denote a sequence of $L + 1$ natural numbers, where $\din = d_0, d_1, \dots, d_{L-1}, d_L = \dout$. With $W$, we denote a sequence of $L$ real matrices, s.t. $W^{(i)} \in \mathbb{R}^{d_i \times d_{i-1}}$, for each $i \in [L]$. Finally, with $Q$ we denote a sequence of $L$ real vectors, s.t. $\mathbf{q}^{(i)} \in \mathbb{R}^{d_i}$, for each $i \in [L]$. For an input $\mathbf{x} \in \mathbb{F}$, the value of $\sigma(\mathbf{x})$ is given as the value $\sigma^{(L)}$ in the system of recursive equations below.
    \begin{equation}
        \label{eq:mlp}
        \left.
        \begin{array}{ll}
             \sigma^{(0)} &= \mathbf{x} \\
             \sigma^{(i)} &= \mathbf{r}[W^{(i)}\sigma^{(i-1)} + \mathbf{q}^{(i)}],\quad \forall i \in [L]
        \end{array}
        \right\}
    \end{equation}
\end{definition}
For \emph{classification} problems, let $\mathcal{C} \subset \mathbb{N}$ be a finite set of classes, with $\abs{\mathcal{C}} = \dout$. A classifier $\kappa \colon \mathbb{F} \to \mathcal{C}$ is constructed, with respect to the MLP $\sigma(\cdot)$, as $\kappa(\mathbf{x}) = \arg \max_{i \in [\dout]} \sigma_i(\mathbf{x})$. For a class $c \in \mathcal{C}$, we denote with $\mathcal{D}_c $ 
the \emph{decision surface} of the class $c$. Namely, $\mathcal{D}_c$ is the pre-image of $\kappa(c)$; consisting of all the inputs in $\mathbb{F}$ that are classified to $c$, by $\kappa(\cdot)$.

\paragraph{Formal MLP Verification.}
The MLP of Def. \ref{def:mlp} can be expressed as a set of linear inequalities, with real and integer variables. This formalization is known in the literature as \emph{Mixed Integer Linear Programming (MILP)}.
\begin{equation}
    \label{eq:milp}
    \left.
    \begin{array}{l l l}
        \multicolumn{2}{l}{\widehat{\mathbf{z}}^{(0)} = \mathbf{x}, ~\mathbf{y} = \widehat{\mathbf{z}}^{(L)}}\\
        \mathbf{z}^{(i)} &= W^{(i)}\widehat{\mathbf{z}}^{(i-1)} + \mathbf{q}^{(i)}, &\forall i \in [L]\\
        \multicolumn{2}{l}{\widehat{\mathbf{z}}^{(i)} \geq \mathbf{z}^{(i)}, ~\widehat{\mathbf{z}}^{(i)} \geq \mathbf{0},} &\forall i \in [L]\\
        \widehat{\mathbf{z}}^{(i)} &\leq \mathbf{z}^{(i)} + M\mathbf{t}^{(i)}, &\forall i \in [L]\\
        \widehat{\mathbf{z}}^{(i)} & \leq M(\mathbf{1} - \mathbf{t}^{(i)}), &\forall i \in [L]\\ \\
        \multicolumn{2}{l}{\mathbf{x} \in \mathbb{R}^{d_\text{in}}, ~\mathbf{y} \in \mathbb{R}^{d_\text{out}}}\\
        \multicolumn{2}{l}{\mathbf{t}^{(i)} \in \{0, 1\}^{d^i_\text{out}}, ~\mathbf{z}^{(i)} \in \mathbb{R}^{d^i_\text{out}}} &\forall i \in [L]\\
    \end{array}
    \right\}
\end{equation}
In eq. \eqref{eq:milp}, we denote with $\mathbf{z}$ the value of the neuron \emph{before} the ReLU activation, while with $\widehat{\mathbf{z}}$ the value of the neuron \emph{after} the activation is applied. We use the constant $M$\footnote{
    Here we use the big-M formalization of \cite{lomuscio}. Other formalizations have also been proposed, see the survey of \cite{meng}.
} representing a high value, practically treated as infinity. The variables $\mathbf{t}$ model ReLU's behaviour. For the $j$-th neuron of the $i$-th layer, $t^{(i+1)}_j = 0$ \emph{iff} $\widehat{z}^{(i+1)}_j = z^{(i+1)}_j$, i.e., ReLU is activated; otherwise, $t^{(i+1)}_j = 1$.

The MILP formalization allows us to \emph{rigorously} reason about the MLP's behaviour. In particular, eq. \eqref{eq:milp} formally defines a relation $\mathcal{N} \subseteq \mathbb{F} \times \mathbb{S}$, s.t. for an I/O-pair $\angles{\mathbf{x}, \mathbf{y}} \in \mathbb{F} \times \mathbb{S}$, we have $\angles{\mathbf{x}, \mathbf{y}} \in \mathcal{N}$, \emph{iff} $\sigma(\mathbf{x}) = \mathbf{y}$. A verifier is essentially another relation $\mathcal{V} \subseteq \mathbb{F} \times \mathbb{S}$. We call the verifier $\mathcal{V}$ \emph{sound} if $\mathcal{V} \subseteq \mathcal{N}$. We call the verifier $\mathcal{V}$ \emph{complete} if $\mathcal{N} \subseteq \mathcal{V}$. Sound and complete verifiers such as Marabou~\cite{marabou} make heavy use of cutting-edge MILP solvers, e.g., Gurobi\footnote{\url{https://www.gurobi.com}}, while utilizing sophisticated heuristics tailored for MILPs modeling MLPs. This allows them to analyse much larger networks, whose size would be otherwise prohibiting. Still, note that verifying a MLP is an \textbf{NP}-complete problem \cite{katz}.

    \section{Intervals in Higher Dimensions}
    \label{sec:intervals}
    In this section, we present some foundational results from \emph{Interval Algebra} \cite{sunaga,moore}. Firstly, we generalize the $\leq ~\subseteq \mathbb{R} \times \mathbb{R}$ relation to high-dimensional spaces. For two vectors $\bell, \bu \in \mathbb{R}^d$ we write $\bell \leq \bu$ \emph{iff} $\ell_i \leq u_i$ for all $i \in [d]$. Similarly, we write $\bell < \bu$ \emph{iff} $\ell_i < u_i$ for each $i \in [d]$\footnote{
    Note that 
    it does \emph{not} hold that $\bell < \bu$ whenever $\bell \leq \bu$ and $\bell \neq \bu$.
}. Below, we describe a generalization of real intervals for high-dimensional spaces.

\begin{definition}[High Dimensional Intervals]
    \label{def:intervals}
    Let $\bell, \bu \in \mathbb{R}^d$, with $\bell \leq \bu$. A \emph{closed interval} $[\bell, \bu] \subset \mathbb{R}^d$ is the set of points $\mathbf{x} \in \mathbb{R}^d$ such that $\bell \leq \mathbf{x} \leq \bu$. An \emph{open interval} $(\bell, \bu)$ is the interior of the respective closed interval, i.e., $(\bell, \bu) = [\bell, \bu]^\circ$. We denote with $\intervals{d}$ the space of the $d$--dimensional \emph{closed} intervals, i.e. $\intervals{d} = \{S \subset \mathbb{R}^d \mid \exists \ \bell, \bu \in \mathbb{R}^d, \ \bell \leq \bu, \ S = [\bell, \bu]\}$.
\end{definition}
From the above definition, it is easy to see that $\mathbf{x} \in (\bell, \bu)$ \emph{iff} $\bell < \mathbf{x} < \bu$. Note that for any $\mathbf{x} \in \mathbb{R}^d$ the closed interval $[\mathbf{x}, \mathbf{x}]$ is a \emph{trivial interval}, corresponding to the singleton $\{\mathbf{x}\}$. Additionally, the trivial open interval $(\mathbf{x}, \mathbf{x})$ corresponds to the empty set $\varnothing$. Observe that the sphere $\mathcal{B}^\infty(\mathbf{x}, \rho)$ corresponds to the \emph{uniform} interval $[\mathbf{x} - \rho\mathbf{1}, \mathbf{x} + \rho\mathbf{1}]$. Geometrically, an interval $[\bell, \bu]$ corresponds to a \emph{hyper-rectangle} in $\mathbb{R}^d$. In particular, a uniform interval centered at $\mathbf{x}$ corresponds to a \emph{hyper-cube} with $\mathbf{x}$ as its barycenter. We extend the notation of Definition \ref{def:intervals}, denoting with $\intervalsdx \subseteq \intervalsd$ \emph{the set of all intervals including the point $\mathbf{x} \in \mathbb{R}^d$}. Naturally, for every $\mathbf{x} \in \mathbb{R}^d$, and every $\rho > 0$, ~$\infball(\mathbf{x}, \rho) \in \intervalsdx$. We often consider an \emph{interval universe}~ $\universe = [\uu, \ou]$, for specific $\uu, \ou$. We denote with $\intervalsdxu \subseteq \intervalsdx$~ \emph{all the intervals that include the point $\mathbf{x}$ and are included in $\universe$}.

\subsection{Operations on Intervals \& the Interval Lattice}
Below we give some elementary operations on the interval space $\intervalsdxu$.
\begin{restatable}{proposition}{intervalops}
    \label{prop:interval-ops}
    Let $[\bell, \bu], [\bm, \bn] \in \intervalsdxu$~ two $d$--dimensional intervals, and the operations:
    \begin{itemize}
        \item $[\bell, \bu] + [\bm, \bn] \overset{\Delta}{=} [\bell + \bm, \bu + \bn]$
        \item $[\bell, \bu] \sqcup [\bm, \bn] \overset{\Delta}{=} [\min\{\bell,\bm\}, \max\{\bu, \bn\}]$
        \item $[\bell, \bu] \sqcap [\bm, \bn] \overset{\Delta}{=} [\max\{\bell,\bm\}, \min\{\bu, \bn\}]$
    \end{itemize}
    For any $\square \in \{+, \sqcup, \sqcap\}$ and $I, J \in \intervalsdxu$, $I ~\square~ J \in \intervalsdxu$.
\end{restatable}
Observe that the $\sqcap$ operator coincides with the set-theoretic intersection $\cap$. However, it holds that $I ~\sqcup~ J \supsetneq I \cup J$. In general, the union of two intervals is not an interval. Below, we review how the $\sqcup, \sqcap$ operations reveal the underlying structure of the interval space $\intervalsdxu$.
\begin{theorem}[Interval Lattice, \cite{sunaga}]
    \label{theo:interval-lattice}
    The interval space $\intervalsdxu$, organized under $\subseteq$ constitutes a complete \emph{lattice} with $\sqcup, \sqcap$ as the \emph{meet} and \emph{join} operations, respectively.
\end{theorem}
Note that the set-theoretic exclusion of a point from an interval ($I \setminus \mathbf{x}$) does not yield an interval. We therefore define an alternative exclusion operator that removes a point $\mathbf{x}$ from an interval $I= [\bell, \bu]$. 
The operator selects a coordinate along which the induced modification to $I$ is minimal, and then adjusts either $\ell_i$ or $u_i$
--which results in the smaller change--by setting it to a value infinitesimally smaller or larger than $x_i$. This infinitesimal offset is formalized using a parameter $\delta>0$. 
\begin{definition}
    \label{def:interval-exclusion}
    Let $I = [\bell, \bu] \in \intervalsdxu$ be an interval and $\mathbf{x}^\prime \in I$ a point included in the interval $I$. Let $k = \arg\max\{\abs{x_i - x^\prime_i} \mid i \in [d] \wedge \mathbf{x} \in \intervalsdxu\}$\footnote{In general, the $\arg\max\{\abs{x_i - x^\prime_i} \mid i \in [d] \wedge \mathbf{x} \in \intervalsdxu\}$ may provide a set of indices, meaning that there are ties w.r.t. the smallest changes we can impose in the interval. Since, in any interval, any change in any dimension is orthogonal to any changes in any other dimension, we can apply any tie-breaking rule, e.g., lexicographic ordering.}. If $x_i - x^\prime_i > 0$, then $I \ExOp{\delta} \mathbf{x}^\prime = [\bell^\prime, \bu]$, where $\ell^\prime_i = \ell_i$, for every $i \neq k$ and $\ell_k = x^\prime_k + \delta$, for some $\delta > 0$. If $x_i - x^\prime_i < 0$, then $I \ExOp{\delta} \mathbf{x}^\prime = [\bell, \bu^\prime]$, where $u^\prime_i = u_i$, for every $i \neq k$ and $u_k = x^\prime_k - \delta$, for some $\delta > 0$.
\end{definition}
Despite its subtleties, we will see in Sec. \ref{sec:approximations} that operator $\ExOp{\delta}$ is natural. $I\ExOp{\delta}\mathbf{x}^\prime$ chooses a \emph{maximum} refinement of $I$ that excludes $\mathbf{x}^\prime$. Further, to simplify notation, we will drop the $\delta$ from the notation of $\ExOp{\delta}$ when it's clear from the context.

\subsection{Interval Objectives}\label{subsec:interval-measures}

To formulate our methodology, we examine the following family of measures on intervals. We call these quantities objectives, since they are optimized in the computation of maximal sound or minimal complete intervals.
\begin{definition}[Interval Objectives]
    \label{def:interval-objectives}
    Consider an interval $I = [\bell, \bu] \in \intervalsdxu$, with $\bell \leq \bu$. Then we define the following measures:
    \begin{center}
        \begin{tabular}{p{11em}  l l}
            Minimum Edge Length:&  $\alpha(I)$ &$= \min_{i \in [d]} u_i - \ell_i$\\
            Perimeter:& $\pi(I)$            &$ = \sum_{i \in [d]} u_i - \ell_i$\\
            Volume:&    $\vol(I)$           &$= \prod_{i \in [d]} u_i - \ell_i$.\\
            Diameter:&  $\mathcal{A}(I)$    &$= \max_{i \in [d]} u_i - \ell_i = \|u_i - \ell_i\|_\infty$
        \end{tabular}
    \end{center}
\end{definition}
These objectives are related through the arithmetic-geometric means inequality.
\begin{restatable}{proposition}{numericalgeometricmean}
    \label{prop:numerical-geometric-mean}
    For an interval $I \in \intervalsdxu$ and the measures of Def. \ref{def:interval-objectives} we have,
    \begin{equation}
        \label{eq:numerical-geometric-mean}
        \mathcal{A}(I) \geq \frac{1}{d}\cdot\pi(I) \geq \sqrt[d]{v(I)}\geq \alpha(I)
    \end{equation}
\end{restatable}
Eq. \eqref{eq:numerical-geometric-mean} highlights the significance of the minimum edge length measure within the family of interval measures defined in Def. \ref{def:interval-objectives}, since it provides an explicit lower bound on all other measures. Thus, it \emph{suffices} to ensure the non-triviality, i.e., strict positivity, of $\alpha(I)$, to ensure the non-triviality of all the remaining objectives. This fact supports the choice of the exclusion operation of Def. \ref{def:interval-exclusion}, since it computes the \emph{optimal} exclusion w.r.t. the minimum edge length objective.
\begin{restatable}{theorem}{intervalexclusion}
    \label{theo:interval-exclusion}
    Consider an interval $I \in \intervalsdxu$ and a point $\mathbf{x}^\prime \in I$. For any interval $J \in \intervalsdxu$, with $\mathbf{x}^\prime \notin J$ and $J \subseteq I$, we have $\alpha(J) \leq \alpha(I/\mathbf{x}^\prime)$.
\end{restatable}

    \section{Sound \& Complete Interval Certifications}
    \label{sec:approximations}
    \begin{figure}[t!]
    \begin{center}
        \includegraphics[scale=0.25]{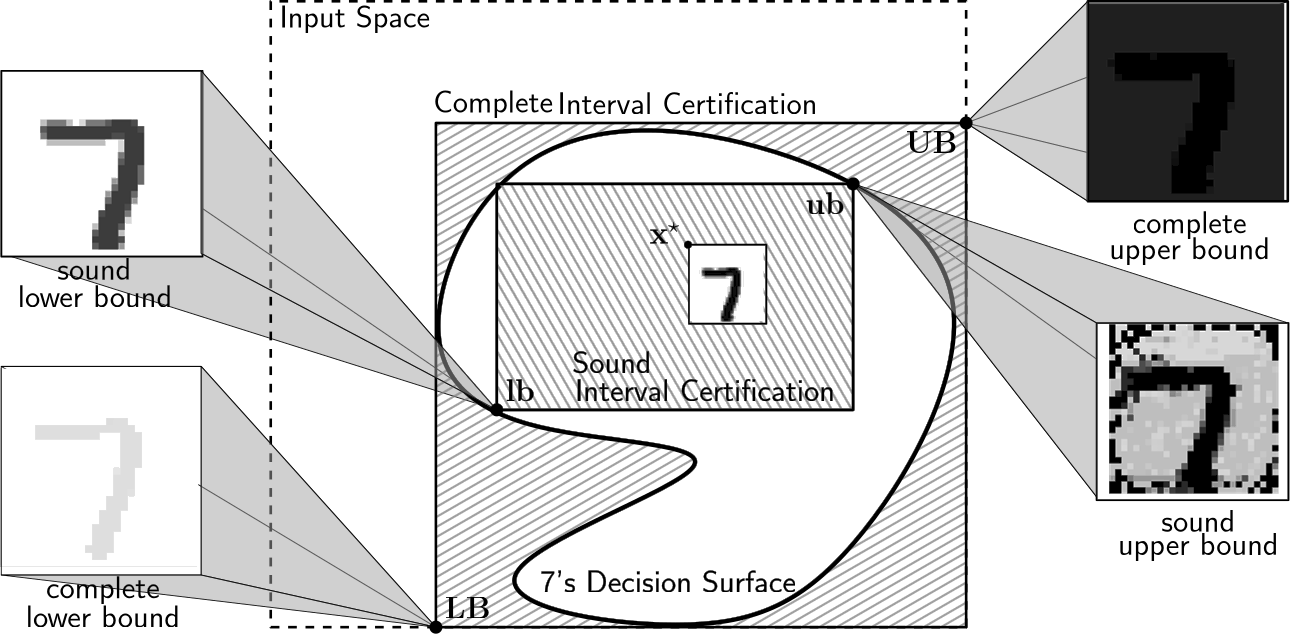}
    \end{center}
    \caption{A sound interval certification $[\mathbf{lb}, \mathbf{ub}]$, and a complete interval certification $[\mathbf{LB}, \mathbf{UB}]$, for a MNIST image of ``7''. The sound interval certification has been computed using the Algorithm \ref{algo:generic-traversal}
    , while its complete counterpart was computed by Algorithm \ref{algo:non-det-expand}
    in $\ell_\infty$-spheres.}
    \label{fig:sevens}
\end{figure}

\begin{definition}
    \label{def:interval-sound-complete}
    Consider a classifier $\kappa \colon \mathbb{F} \to \mathcal{C}$, with $\mathbb{F} \subset \mathbb{R}^{d}$, and $\mathbf{x} \in \mathbb{F}$ an input, s.t. $\kappa(\mathbf{x}) = c$. Moreover, let $I, J \in \mathbb{I}(d)$ two intervals, s.t.
    \begin{equation}
        \label{eq:interval-sound-complete}
    \resizebox{0.931\hsize}{!}{
    $
        \mathbf{x} \in I \land \left[\forall \mathbf{x}^\prime \in \mathbb{F}\colon \ \mathbf{x}^\prime \in I \to \kappa(\mathbf{x}^\prime) = c \right], \enspace
        \mathbf{x} \in J  \land \left[\forall \mathbf{x}^\prime \in \mathbb{F}\colon \ \mathbf{x}^\prime \notin J \to \kappa(\mathbf{x}^\prime) \neq c \right],
    $
    }
    \end{equation}
    then we call $I$ a \emph{sound} and $J$ a complete \emph{certification}.
\end{definition}
Intuitively, a sound certification $I$ assures that any perturbation of the input $\mathbf{x}$ inside $I$ will \emph{not} change its prediction. On the other hand, a complete certification $J$ asserts that if $\mathbf{x}$ is moved outside of $J$, then its prediction is \emph{guaranteed} to change (see Fig. \ref{fig:sevens}). Naturally, a \emph{trivial} sound certification is the input itself, expressed as an interval, i.e., $I = [\mathbf{x}, \mathbf{x}]$. Similarly, a trivial complete certification will be the entire input space $\mathbb{F}$ itself. Therefore, we are interested in \emph{maximally} sound and \emph{minimally} complete certifications. Since $I, J$ are subsets of $\mathbb{F}$, maximality and minimality are considered w.r.t.\ set inclusion.

\subsection{Verification Oracles}

In our algorithms, we will utilize two oracles that either verify the truth of a property on a given interval, or return a \emph{counterexample}, witnessing its falsehood. In particular, we are interested in \emph{soundness} and \emph{completeness} oracles, verifying an interval's respective property. These oracles can be constructed using a sound and complete MLP verifier\footnote{
    Note that soundness and completeness of a MLP verifier, w.r.t. 
    Sec.\ref{sec:preliminaries} differ from soundness and completeness of an interval certification w.r.t. Def. \ref{def:interval-sound-complete}.
}. Indeed, we make use of the Marabou verifier \cite{marabou}, where we \emph{extend} the set of constraints fed to the verifier. To that end, let $\mathcal{N}(\mathbf{x}, \mathbf{y})$ be the I/O-relation describing a MLP, as in eq. \eqref{eq:milp}.

For a given MLP, described by the relation $\mathcal{N}(\cdot, \cdot)$ and an interval $I = [\bell, \bu] \subseteq \mathbb{F}$, the \emph{soundness oracle} is given by the predicate $\mathscr{S}_{I, \mathcal{N}}$, where,
\begin{equation}
    \label{eq:sound-oracle}
    \resizebox{0.93\hsize}{!}{
    $
    \mathscr{S}_{I, \mathcal{N}}(\mathbf{x}^\prime) \equiv \exists\mathbf{y}.
        ~(\bell \leq \mathbf{x}^\prime \leq \bu)~ \land
        ~\mathcal{N}(\mathbf{x}^\prime, \mathbf{y})~ \land
        ~\left( \bigvee_{j \in [\dout] \setminus c} y_j - y_{c} > \epsilon \right).
    $
    }
\end{equation}
If the predicate $\mathscr{S}_{I, \mathcal{N}}(\mathbf{x}^\prime)$ is a \emph{contradiction}, i.e., $\models \lnot \mathscr{S}_{I, \mathcal{N}}(\mathbf{x}^\prime)$, the interval $I$ is a sound certification. However, any $\mathbf{x} \in \mathbb{F}$, s.t. $\mathbf{x}^\prime \models \mathscr{S}_{I, \mathcal{N}}(\mathbf{x}^\prime)$ is a \emph{counterexample}. In the literature, such an input is also called \emph{adversarial example} or \emph{adversarial attack}. Such a counterexample $\mathbf{x}^\prime$ can \emph{``fool''} the MLP, since it will be relatively close to a given input $\mathbf{x}$, but is classified to a different class from $c = \kappa(\mathbf{x})$. The latter is ensured by the last clause in eq. \eqref{eq:sound-oracle}, by asserting that a score different than $c$ is the highest in the output scores vector, w.r.t. a constant $\epsilon > 0$.

Conversely, the \emph{completeness oracle} is provided by predicate $\mathscr{C}_{I, \mathcal{N}}$, where,
\begin{equation}
    \label{eq:complete-oracle}
    \resizebox{0.93\hsize}{!}{
    $
    \mathscr{C}_{I, \mathcal{N}}(\mathbf{x}^\prime) \equiv 
        \left( \bigvee_{i \in [\din]} x_i^\prime < \ell_i  \lor u_i < x_i^\prime \right)~ \land ~\mathcal{N}(\mathbf{x}^\prime, \mathbf{y})~ \land ~\left( \bigwedge_{j \in [\dout] \setminus c} y_{c} - y_j > \epsilon \right).
    $
    }
\end{equation}
If the predicate $\mathscr{C}_{I, \mathcal{N}}(\mathbf{x}^\prime)$ is a \emph{contradiction}, i.e., $\models \lnot \mathscr{C}_{I, \mathcal{N}}(\mathbf{x}^\prime)$, then the interval $I$ is a complete certification. However, any $\mathbf{x} \in \mathbb{F}$, s.t.\ $\mathbf{x}^\prime \models \mathscr{C}_{I, \mathcal{N}}(\mathbf{x}^\prime)$ is a \emph{counterexample}. 
Essentially, the completeness predicate is the dual of $\mathscr{S}_{I, \mathcal{N}}$. 
Notably, since we use a sound and complete verification system, invoking both the oracles $\mathscr{S}_{I, \mathcal{N}}$ and $\mathscr{C}_{I, \mathcal{N}}$ inherits the computational hardness of MLP verification \cite{katz}.

Finally, observe that the formulas in eq. \eqref{eq:sound-oracle} and \eqref{eq:complete-oracle} can be use to \emph{``locally''} describe the MLP. Consider a sound certification $I$ and a complete verification $J$, then it holds $\lnot \mathscr{I}_{I, \mathcal{N}}(\mathbf{x}) \models \kappa(\mathbf{x}) = c \models \lnot \mathscr{C}_{J, \mathcal{N}}(\mathbf{x})$, since $J \supseteq \mathcal{D}_c \supseteq I$. Ensuring completeness minimality and soundness maximality minimizes the gap between $J, I$ and provides a \emph{tighter} local approximation of the underlying MLP.

\subsection{The Generic Traversal Algorithm}

In Algorithm \ref{algo:generic-traversal} we present a simple traversal algorithm that explores the interval lattice $\intervalsdxu$ introduced in Sec. \ref{sec:intervals}. Our algorithm is generic in the sense that it can compute interval certifications with different properties depending on the parameters passed as arguments. In the sequel, we show how to implement three \emph{big-step}\footnote{Borrowing static analysis terminology.} operators on intervals. Each big step operator is the result of iteratively applying one of the \emph{small-step} operators introduced in Sec. \ref{sec:intervals}, i.e., $\{+, \sqcup, \sqcap, /\}$.

\begin{algorithm}[t]
    \DontPrintSemicolon
    \caption{\textsc{GenericTraversal}}
    \label{algo:generic-traversal}
    \KwInput{$I_0 \in \intervalsdxu$, an initial interval, $\varphi_{I, \mathcal{N}}$, a property to be falsified, and $\square\colon \intervalsdxu \times \mathbb{F} \to \intervalsdxu$, a refinement operator.}
    \KwOutput{$I$,  an interval s.t. $\models \lnot \varphi_{I, \mathcal{N}}$}

    $I \gets I_0$ 
    
    \While{$\exists \mathbf{x}^\prime$, s.t. $\mathbf{x}^\prime \models \varphi_{I, \mathcal{N}}$}{
        $I \gets I ~\square~\mathbf{x}^\prime$
    }
\end{algorithm}

\subsection{Minimally Complete Certifications via the $\top$--Operator}\label{subsec:minimally-complete}

We begin our discussion on \emph{maximally complete certifications} by defining the $\top$-operator (see Def. \ref{def:top-operator}), between a \emph{compact} decision surface $\mathcal{D}_c \subseteq \mathbb{F}$ and an input point $\mathbf{x} \in \mathbb{F}$, s.t. $\kappa(\mathbf{x}) = c$. We see that we can obtain a maximally complete interval certification, w.r.t. $\mathcal{D}_c$ and $\mathbf{x}$ as a result of the $\top$-operator. Moreover, the complete interval certification is \emph{unique} for a particular class, and independent from the choice of the input $\mathbf{x}$. We show this fact in Th. \ref{theo:complete-certification}.
\begin{definition}[$\top$-Operator]
    \label{def:top-operator}
    Consider a classifier $\kappa\colon\mathbb{F} \to \mathcal{C}$, an input point $\mathbf{x} \in \mathbb{F}$, s.t. $\kappa(\mathbf{x}) = c$, and the compact decision surface $\mathcal{D}_c$. With $\mathcal{D}_c\top\mathbf{x}$ we denote the interval $[-\underline{\mathbf{R}}, \overline{\mathbf{R}}]$, s.t.,
    \begin{equation}
        \label{eq:Radii}
            \resizebox{0.93\hsize}{!}{
    $ \uRadi = \max\{x_i - x^\prime_i \mid x_i^\prime \leq x_i, \ \mathbf{x}^\prime \in \mathcal{D}_c\}, \ \oRadi = \max\{x^\prime_i - x_i \mid x_i^\prime > x_i, \ \mathbf{x}^\prime \in \mathcal{D}_c\}.$
    }
    \end{equation}
\end{definition}
\begin{restatable}{theorem}{completecertification}
    \label{theo:complete-certification}
    The interval $\mathcal{D}_c\top\mathbf{x}$ is the \emph{unique} complete certification. Moreover, for any two $\mathbf{x}_1, \mathbf{x}_2 \in \mathcal{D}_c$, we have $\mathcal{D}_c\top\mathbf{x}_1 = \mathcal{D}_c\top\mathbf{x}_2$.
\end{restatable}

The interval $\mathcal{D}_c\top\mathbf{x}$ can be computed using Algorithm \ref{algo:generic-traversal}, with the parameters $I_0 = [\mathbf{x}, \mathbf{x}]$, $\varphi_{I, \mathcal{N}} = \mathscr{C}_{I, \mathcal{N}}$, and $I ~\square~ \mathbf{x}^\prime = I \sqcup [\mathbf{x}^\prime -\delta\mathbf{1}, \mathbf{x}^\prime +\delta\mathbf{1}]$, where $\delta > 0$ is a precision constant. Intuitively, with these parameters Algorithm \ref{algo:generic-traversal} implements a \emph{bottom-up} search policy (\textsc{BUS}), beginning from the trivial interval $[\mathbf{x}, \mathbf{x}]$, proceeding to greater intervals, by including members of the surface $\mathcal{D}_c$. The precision parameter $\delta > 0$ regulates both the outcome's accuracy and the convergence rate of our algorithm.
\begin{restatable}{theorem}{algomincomplete}
    \label{theo:algo-min-complete}
    Consider a compact decision surface $\mathcal{D}_c \subset \mathbb{F}$. Let $I \in \intervalsdxu$ the interval returned by Algorithm \ref{algo:generic-traversal}. It holds that $\mathcal{D}_c\top\mathbf{x} \subseteq I \subseteq \mathcal{D}_c\top\mathbf{x} + [-\delta\mathbf{1}, \delta\mathbf{1}]$. Moreover, Algorithm \ref{algo:generic-traversal} terminates after $O(d\cdot\mathcal{A}(\mathcal{D}_c\top\mathbf{x})/\delta)$ steps.
\end{restatable}

\subsection{Edge Length Non-Triviality in Sound Certifications via the $\bot$-Operator}

Before we proceed to maximally sound certifications, we present the $\bot$-operator that returns a special kind of sound certifications. These certifications are guaranteed (under certain assumptions) to have strictly positive minimum edge length. Let $\mathbf{x} \in \mathbb{F}$ be an input point. We consider the sequence of polyhedral cones\footnote{
    \label{foot:geometry}
    A polyhedral cone is a set of the form $C = \{\mathbf{x} \in \mathbb{R}^n \mid A\mathbf{x} \leq \mathbf{0}\}$, where $A \in \mathbb{R}^{m \times n}$. We briefly explore the geometry of the cones and their intersections in Appendix~\ref{app:geometry}.
} $\angles{\underline{V}_i, \overline{V}_i}$ for each $i \in [d]$, in eq. \eqref{eq:cones}. We have $\underline{V}_i \cap \overline{V}_i = \{\mathbf{0}\}$, for every $i \in [d]$, and $(\cup_{i \in [d]} \underline{V}_i) \cup (\cup_{i \in [d]} \overline{V}_i) = \mathbb{R}^d$. Using the cone partition of eq. \eqref{eq:cones} we define the $\bot$-operator.
\begin{equation}
    \label{eq:cones}
    \begin{split}
    & \underline{V}_i = \left\{\mathbf{x}^\prime \in \mathbb{F}~\big\vert~ x^\prime_i \leq x_i, \forall j \in [d] ~\abs{x^\prime_i} \geq \abs{x^\prime_j} \right\}, \\
    & \overline{V}_i = \left\{\mathbf{x}^\prime \in \mathbb{F}~\big\vert~ x^\prime_i \geq x_i, \forall j \in [d] ~\abs{x^\prime_i} \geq \abs{x^\prime_j} \right\}.
    \end{split}
\end{equation}

\begin{definition}[$\bot$-Operator]
    \label{def:bot-operator}
    Consider a classifier $\kappa\colon\mathbb{F} \to \mathcal{C}$, an input point $\mathbf{x} \in \mathbb{F}$, s.t. $\kappa(\mathbf{x}) = c$, and the compact decision surface $\mathcal{D}_c$. With $\mathcal{D}_c\bot\mathbf{x}$ we denote the interval $[-\underline{\mathbf{r}}, \overline{\mathbf{r}}]$, s.t.,
    \begin{equation}
        \label{eq:radii}
        \uradi = \inf\{x_i - x^\prime_i \mid \mathbf{x}^\prime \in \underline{V_i} \setminus \mathcal{D}_c\}, \quad \oradi = \inf\{x^\prime_i - x_i \mid \mathbf{x}^\prime \in \overline{V_i} \setminus \mathcal{D}_c\}.
    \end{equation}
\end{definition}
In Prop. \ref{prop:maximal-sound-approx} we establish the soundness of $\mathcal{D}_c\bot\mathbf{x}$.

\begin{restatable}{proposition}{propmaximalsoundapprox}
    \label{prop:maximal-sound-approx}
    Let $\mathcal{D}_c \subseteq \mathbb{F}$ be a compact decision surface, s.t. $\mathbf{x} \in \mathcal{D}_c$. It holds $\mathcal{D}_c\bot\mathbf{x} \subseteq \mathcal{D}_c$.
\end{restatable}

The interval $\mathcal{D}_c\bot\mathbf{x}$ can be computed using Algorithm \ref{algo:generic-traversal}, with the parameters $I_0 = \mathbb{F}$, $\varphi_{I, \mathcal{N}} = \mathscr{S}_{I, \mathcal{N}}$, and $I ~\square~ \mathbf{x}^\prime = I \ExOp{\delta} \mathbf{x}^\prime$, as in Def. \ref{def:interval-exclusion}. 
Algorithm \ref{algo:generic-traversal} implements a \emph{top-down} search policy (\textsc{TDS}), beginning from the whole input space $\mathbb{F}$, proceeding to smaller intervals, by excluding counterexamples $\mathbf{x}^\prime \in \mathbb{F} \setminus \mathcal{D}_c$. Similar to Subsec.~\ref{subsec:minimally-complete}, the precision parameter $\delta$ regulates both the outcome's accuracy and the convergence rate of our algorithm.

\begin{restatable}{theorem}{theoalgomaxsound}
    \label{theo:algo-max-sound}
    Consider a compact decision surface $\mathcal{D}_c \subset \mathbb{F}$. Let $I \in \intervals{d}\vert^\universe_\mathbf{x}$ the interval returned by Algorithm \ref{algo:generic-traversal}. It holds that $\mathcal{D}_c\bot\mathbf{x} - [-\delta\mathbf{1}, \delta\mathbf{1}] \subseteq I \subseteq \mathcal{D}_c$. Moreover, the algorithm terminates after $O(d[\mathcal{A}(\universe) - \mathcal{A}(\mathcal{D}_c\bot\mathbf{x})]/\delta)$ steps.
\end{restatable}

The following result guarantees a computed certification with non-trivial edge lengths, provided that the 
input belongs to the interior of the decision surface. 
\begin{restatable}{proposition}{nontrivialitylb}
    \label{prop:non-triviality-lb}
    If $\mathbf{x} \in \mathcal{D}^\circ_c$, then $\alpha(\mathcal{D}_c\bot\mathbf{x}) > 0$.
\end{restatable}
Recall that, from Prop. \ref{prop:numerical-geometric-mean}, several objectives are lower bounded by the minimum edge length. Thus, the interval $\mathcal{D}_c\bot\mathbf{x}$ provides a non-trivial lower bound to these quantities, as well. Notably, Prop. \ref{prop:non-triviality-lb} suggests that we can verify a minimum coordinate-wise perturbation tolerance. Our experiments in Sec. \ref{sec:implementation} show that this also holds in practice. We discuss optimization aspects of interval certifications at the end of this section.

\subsection{The Maximal Closure Operator $[I]$}

Computing maximally sound certifications share similarities with their complete counterparts, but also exhibit important differences. The primary difference is that, in general, there will be multiple maximal sound certifications. Therefore, for a sound interval $I$, we consider its maximal closure, denoted with $[I]$, that contains all the supersets of $I$ that are maximally sound.

\begin{definition}[Maximal Closure Operator]
    \label{def:max-closure}
     Consider a classifier $\kappa\colon\mathbb{F} \to \mathcal{C}$, an input point $\mathbf{x} \in \mathbb{F}$, s.t. $\kappa(\mathbf{x}) = c$, and the compact decision surface $\mathcal{D}_c$. Let $I \subseteq \mathcal{D}_c$ be a sound interval certification, s.t. $\mathbf{x} \in I$. With $[I]$ we denote the \emph{maximal closure} of $I$, where, $[I] = \{J \in \intervalsdxu \mid I \subseteq J \subseteq \mathcal{D}_c\}.$
\end{definition}

Lem. \ref{lem:maximal-soundness} characterises all the maximally sound interval certifications. Intuitively, an interval certification is maximally sound if expanding unilaterally any of its coordinates results in the inclusion of a counterexample. Based on Lem. \ref{lem:maximal-soundness}, we present a non-deterministic algorithm, which \emph{chooses} a coordinate of each of the two interval's endpoints to expand in each step. If the expansion leads to a non-sound certification, the coordinate is discarded, not to be expanded any further. The pseudocode of the algorithm is presented in App. \ref{app:dichotomic}, while its correctness is established in Th. \ref{theo:non-det-expand}.

\begin{restatable}{lemma}{lemmaximalsoundness}
    \label{lem:maximal-soundness}
     Consider a classifier $\kappa\colon\mathbb{F} \to \mathcal{C}$, an input point $\mathbf{x} \in \mathbb{F}$, s.t. $\kappa(\mathbf{x}) = c$, and the compact decision surface $\mathcal{D}_c$. Moreover, let $I = [\bell, \bu]  \subseteq \mathcal{D}_c$. The interval $I$ is \emph{maximally sound}, iff for every $i \in [d]$, and every $\delta > 0$, we have $[\bell- \delta\mathbf{e}_i, \bu] \setminus \mathcal{D}_c \neq \varnothing$ and $[\bell, \bu + \delta\mathbf{e}_i] \setminus \mathcal{D}_c \neq \varnothing$.
\end{restatable}

\begin{restatable}{theorem}{nondetexpand}
    \label{theo:non-det-expand}
    Consider a compact decision surface $\mathcal{D}_c \subset \mathbb{F}$, and $\mathbf{x} \in \mathbb{F}$ an input. Let $X \in \intervalsdxu$ be the trivial sound certification $X = [\mathbf{x}, \mathbf{x}]$. For every maximally sound interval $I \in [X]$ there is a choice of indices in the operation of Algorithm \ref{algo:non-det-expand}, resulting to an interval $J$, s.t. $I - [-\delta\mathbf{1}, \delta\mathbf{1}] \subseteq J \subseteq I$. Moreover, such $J$ is computed in $O(d\cdot\mathcal{A}(\mathbb{F})/\delta)$ \emph{non-deterministic} steps.
\end{restatable}

\subsection{The Space of Interval Certifications}

We close our study on interval certifications by examining the structure of the interval certifications' space and how our earlier discussion is reflected in that space. As we saw in Sec. \ref{sec:intervals}, the structure $\angles{\intervalsdxu, \subseteq}$ forms an \emph{innumerably infinite} lattice \cite{sunaga}, with $\sqcup, \sqcap$ as the join and meet operators respectively. The \emph{bottom element} of this space is the trivial interval $[\mathbf{x}, \mathbf{x}]$ and the top element is the universe $\universe$. Consider a compact decision surface $\mathcal{D}_c \subset \mathbb{F}$, s.t. $\mathbf{x} \in \mathcal{D}_c$, i.e. $\kappa(\mathbf{x}) = c$. The following (set-theoretic) inequality holds: $[\mathbf{x}, \mathbf{x}] \subseteq \mathcal{D}_c\bot\mathbf{x} \subseteq \mathcal{D}_c \subseteq \mathcal{D}_c\top\mathbf{x} \subseteq \universe$. Naturally, every subset of $\mathcal{D}_c\bot\mathbf{x}$ is a sound certification. The interval $\mathcal{D}_c\bot\mathbf{x}$ is not maximal, but every set in its maximal closure $[\mathcal{D}_c\bot\mathbf{x}]$ is. Moreover, \emph{every} complete certification includes $\mathcal{D}_c\top \mathbf{x}$. Fig. \ref{fig:intervals-structure} illustrates the structure of the $\angles{\intervalsdxu, \subseteq}$ lattice. Observe that the hierarchy in Fig. \ref{fig:intervals-structure} collapses when $\mathcal{D}_c$ is an interval, i.e., when $\mathcal{D}_c \in \intervalsdxu$. This fact follows from 
Th.s \ref{theo:complete-certification} and \ref{theo:non-det-expand}.
\begin{corollary} \label{cor:collapsed-hierarchy}
    $[\mathcal{D}_c\bot\mathbf{x}] = \{\mathcal{D}_c\top\mathbf{x}\}$ \emph{iff} $\mathcal{D}_c \in \intervalsdxu$.
\end{corollary}

\begin{figure}[t]
    \centering
    \includegraphics[scale=0.2]{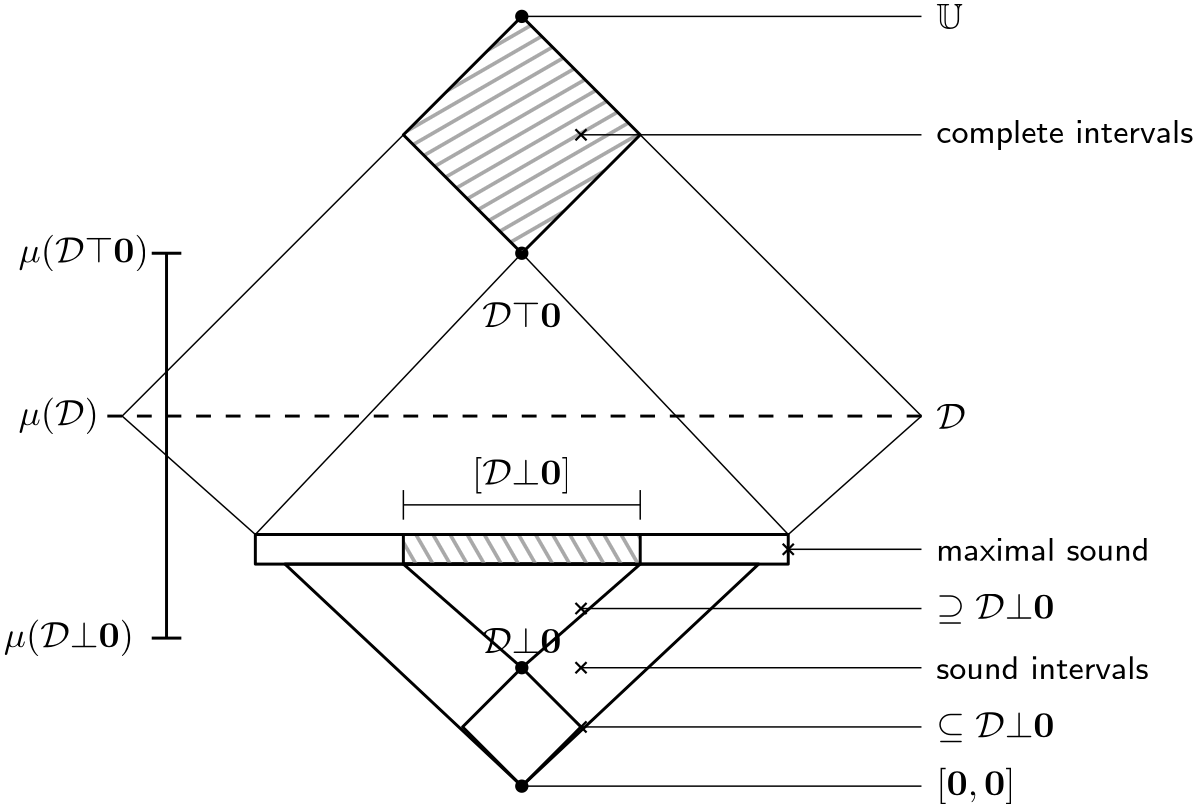}
    \caption{The structure of $\intervals{d}\vert^\universe_{\mathbf{x}}$. We adopt the Hasse diagram convention depicting each element higher from the elements it dominates. $\mathcal{D}_c$ is shown in the image with a dashed line, because, in general, $\mathcal{D}_c \notin \intervals{d}\vert^\universe_{\mathbf{x}}$, and thus, technically, $\mathcal{D}$ is not part of the diagram.  
    Moreover, $\metric(\mathcal{D}_c\bot\mathbf{x}) \leq \metric(\mathcal{D}_c) \leq \metric(\mathcal{D}_c\top\mathbf{x})$, for an increasing objective $\metric(\cdot)$.}
    \label{fig:intervals-structure}
\end{figure}

    \section{Optimization over Interval Certifications}
    \label{sec:optimization}
    We are able to exploit the structure of the interval certification space to compute an optimal certification, w.r.t a certain objective measure and under certain assumptions. To that end, consider an arbitrary objective $\metric\colon\intervalsdxu \to \mathbb{R}_{\geq 0}$, respecting the following properties:
\begin{enumerate}
    \item $\metric([\mathbf{x}, \mathbf{x}]) = 0$
    \item For every $I_1, I_2 \in  \intervalsdxu$, if $I_1 \subseteq I_2$, then $\metric(I_1) \leq \metric(I_2)$.
\end{enumerate}
Note that the objectives of Sec. \ref{sec:intervals} follow these properties. In the sequel, we implicitly assume that $\omega \in \{\alpha, \pi, \vol, \mathcal{A}\}$. Consider the interval certifications $I, J \in \intervalsdxu$, where $I$ is sound, and $J$ is complete, w.r.t. a decision surface $D_c$. For an objective satisfying the above properties, it holds $\metric(I) \leq \metric(\mathcal{D}_c) \leq \metric(J)$. Now, consider the following \emph{dual} optimization problems.
\begin{subequations}
    \begin{align}
        & \underline{\metric} = \max \left\{\metric(I) \mid I \subseteq \mathcal{D}_c \wedge  I \in \intervalsdxu\right\}, \label{eq:interval-optimisation-sim} \\
        & \overline{\metric} = \min \left\{\metric(I) \mid I \supseteq \mathcal{D}_c \wedge  I \in \intervalsdxu\right\}. \label{eq:interval-optimisation-cim}
    \end{align}
\end{subequations}
We call eq. \eqref{eq:interval-optimisation-sim}, the \textsc{Sound Interval Maximization (SIM)} problem, while eq. \eqref{eq:interval-optimisation-cim} \textsc{Complete Interval Minimization (CIM)} problem. Using the $\top$-operator we can \emph{exactly} compute the \emph{optimal} solution $\overline{\metric}$ of the complete minimization problem, i.e. $\overline{\metric} = \metric(\mathcal{D}_c\top\mathbf{x})$, for \emph{any} $\mathbf{x} \in \mathcal{D}_c$. However, the $\bot$-operator only provides a lower bound to the solution of sound maximization, i.e., $\underline{\metric} \geq \metric(\mathcal{D}_c\bot\mathbf{x})$. A better lower bound can be obtained by taking the maximal closure $[\mathcal{D}_c\bot\mathbf{x}]$. Nevertheless, finding the maximum sound interval certification is computationally intractable. Subsequently, we discuss the computational hardness of the sound maximization problem.

\subsection{On the Intractability of Sound Maximization}\label{subsec:intractability of sound maximization}

Recent results provide evidence on the intractability of sound maximization. In particular, recall that \cite{katz} establishes the \emph{NP-completeness} of the MLP verification problem. Since we use verification oracles, this hardness is inherited to our problem. Moreover, \cite{fast-lin} demonstrates the \emph{in-approximability} of sound maximization, even when restricted to uniform interval certifications, i.e., $\ell_\infty$-spheres. Building on these existing results, we further show that the decision version of the 
 \textsc{SIM} problem is \textbf{NP}-hard. Specifically, even when the verification oracle cost is ignored, the problem admits no polynomial-time algorithm, unless $\mathbf{P} = \mathbf{NP}$.

\begin{restatable}{theorem}{theosoundmaxhard}
    \label{theo:sound-max-hard}
    Consider a compact decision surface $\mathcal{D}_c$ and an input point $\mathbf{x} \in \mathcal{D}_c$. Then, the existence of a sound interval $I \in \intervalsdxu$, with $I \subseteq \mathcal{D}_c$ and $\vol(I) \geq \gamma$, $\gamma > 0$ \emph{cannot} be decided in $\mathsf{poly}(d, n, t_{\text{oracle}})$ time, unless \textbf{P}=\textbf{NP}.
\end{restatable}
In Th. \ref{theo:sound-max-hard}, with $\mathsf{poly}(d, n, t_{\text{oracle}})$ we denote the set of all polynomials on $d, n, t_{\text{oracle}}$, where $n$ is the number of oracle calls, and $t_{\text{oracle}}$ is the worst-case time consumed by the soundness verification oracle. Note that Th. \ref{theo:sound-max-hard} constitutes a strong intractability result, since it states that the complexity-significant parameters \emph{cannot} be related polynomially. In other words, it does not suffice to keep the number of dimensions $d$ small, make few oracle calls (i.e., keep $n$ small), or implement faster verifiers (i.e., keep $t_{\text{oracle}}$ small); to efficiently and optimally solve maximum soundness. All three parameters must be kept bounded for an optimal sound certification, to be computationally possible. Thusly, we chose to examine a more efficient \emph{greedy} variant in Sec. \ref{sec:approximations}, based on the 
exclusion operator in Def. \ref{def:interval-exclusion}. Unfortunately, making the optimal decision in every (small) step does not imply overall (big-step) optimality.

The intractability of the maximum soundness problem follows from the intractability of the \textsc{Maximum Empty Rectangle (MER)} problem \cite{chan,empty-boxes,chazelle,naamad}. In this problem, we assume a bounding interval $\mathbb{F}$ and a finite set of \emph{forbidden} points $\mathcal{F} \subset \mathbb{F}$. Our goal is to find an interval $I \subseteq \mathbb{F}$, s.t. $I \cap \mathcal{F} = \varnothing$. Namely, the interval $I$ will exclude all the forbidden points. Moreover, we require $I$ to be the maximum, w.r.t. the volume $v(\cdot)$. \cite{empty-boxes} prove that this problem is \textbf{NP}-hard for arbitrary number of dimensions. However, \textsc{MER} cannot be \emph{directly} reduced to \textsc{SIM}. Note that in \textsc{SIM}, we require the given input $\mathbf{x}$ to be \emph{included} in the returned interval $I$. In literature, the latter problem is known as q-\textsc{MER} \cite{q-mer}, where we demand the maximum empty rectangle to include a query point. Nevertheless, \cite{q-mer} does not provide any intractability result. \cite{empty-boxes} prove \textsc{MER}'s intractability, by reducing from the \textsc{Independent Set (IS)}. The \emph{same} reduction can be applied to show, in Lem. \ref{lem:sound-max-hard}, that \textsc{IS} is \emph{also} reduced to q-\textsc{MER}, by setting $\mathbf{q} = \mathbf{0}$.
\begin{restatable}{lemma}{lemsoundmaxhard}
    \label{lem:sound-max-hard}
    Consider an interval universe $\mathbb{F} = [\underline{\mathbf{U}}, \overline{\mathbf{U}}]$. Moreover assume a finite set of forbidden points $\mathcal{F} \subset \mathbb{F}$, and a query point $\mathbf{q} \in \mathbb{F} \setminus \mathcal{F}$. The existence of an interval $I \subseteq \mathbb{F}$, s.t. $\mathbf{q} \in I$, $I \cap \mathcal{F} = \varnothing$, and $\vol(I) \geq \gamma$, for given $\gamma > 0$, \emph{cannot} be decided in polynomial time, unless \textbf{P}=\textbf{NP}.
\end{restatable}
Th. \ref{theo:sound-max-hard} is derived from Lem. \ref{lem:sound-max-hard} by considering as $\mathcal{F}$ the set of \emph{counterexamples} returned by the soundness verification oracle $\mathscr{S}_{I, \mathcal{N}}$. Finally, from Prop. \ref{prop:numerical-geometric-mean} for each interval it holds $\sqrt[d]{\vol(I)} \geq \alpha(I)$, therefore Lem. \ref{lem:sound-max-hard} and Th. \ref{theo:sound-max-hard} can be modified to use 
$\alpha(\cdot)$, as the desired targeted objective.

\subsection{Sound Maximization on Uniform Intervals}\label{subsec:sound maximiization on symmetric intervals}

We call the \textsc{SIM} problem, when constrained to uniform intervals ($\ell_\infty$-spheres) $\mathbb{B}$-\textsc{SIM}. For disambiguation, we call $\mathbb{I}$-\textsc{SIM}, the \textsc{SIM} problem for general intervals. As we mentioned earlier, $\mathbb{B}$-\textsc{SIM}'s intractability persists, even if constrained to uniform intervals, \cite{katz,fast-lin}. However, we \emph{can} compute an optimal solution in $\mathsf{poly}(d, n, t_{\text{oracle}})$ time. In particular, we can compute a maximum sound uniform interval $B = [\mathbf{x} - \rho\mathbf{1}, \mathbf{x} + \rho\mathbf{1}]$, s.t. $\vol(B) \geq \gamma, \gamma > 0$, in $\log(\mathcal{A}(\mathbb{F}/\delta))\cdot t_{\text{oracle}}$ time. We find the uniform interval's radius $\rho > 0$, by applying \emph{dichotomic search} on the real interval $[0, \mathcal{A}(\mathbb{F})]$. Formally, this is provided by the following result.
\begin{restatable}{theorem}{soundmaxcyclic}
    \label{theo:sound-max-cyclic}
    Consider a compact decision surface $\mathcal{D}_c$ and an input point $\mathbf{x} \in \mathcal{D}_c$. Then, the existence of a sound uniform interval $B = [\mathbf{x} - \rho\mathbf{1}, \mathbf{x} + \rho\mathbf{1}]$, with $\mathbf{x} \in B$, $B \subseteq \mathcal{D}_c$ and $v(B) \geq \gamma$, $\gamma > 0$ \emph{can} be decided in polynomial verification oracle calls.
\end{restatable}
We close our discussion on $\mathbb{B}$-\textsc{SIM}'s complexity, with some remarks on its relation with \textsc{MER}. In particular, \cite{empty-boxes} discusses a variant of \textsc{MER}, restricted to axis-aligned \emph{hyper-cubes}, called \textsc{Maximum Empty Square (MES)}. They also show that \textsc{MES} is \textbf{NP}-hard, for an arbitrary number of dimensions. However, in $\mathbb{B}$-\textsc{SIM}, we are \emph{given} the center of the $\ell_\infty$-sphere. Following our terminology thus far, we would call the latter problem q-\textsc{MES}. Observe that despite \textsc{MER} and q-\textsc{MER} being computationally equal (both \textbf{NP}-hard), q-\textsc{MES} is \emph{properly} easier than \textsc{MES}. Indeed, the dichotomic search method we described in Th. \ref{theo:sound-max-cyclic} would also work for q-\textsc{MES}.

\subsection{The Space of Uniform Interval Certifications}

Let $\mathbb{B}(d)\vert_{\mathbf{x}}^{\mathbb{F}}$ be the set of all uniform intervals centered at $\mathbf{x}$. First we observe that  $\mathbb{B}(d)\vert_{\mathbf{x}}^{\mathbb{F}} \subset \intervalsdxu$, since every uniform interval, is also an interval containing $\mathbf{x}$. However, the opposite does \emph{not} hold. Moreover, $\mathbb{B}(d)\vert_{\mathbf{x}}^{\mathbb{F}}$ exhibits a more straightforward structure than $\intervalsdxu$. Indeed, $\mathbb{B}(d)\vert_{\mathbf{x}}^{\mathbb{F}}$ form a \emph{totally ordered set}, under set inclusion. Namely, for every $B_1, B_2 \in \mathbb{B}(d)\vert_{\mathbf{x}}^{\mathbb{F}}$, we have $B_1 \subseteq B_2$, iff $\rho_1 \leq \rho_2$, where $\rho_1, \rho_2$ the radii of $B_1, B_2$, respectively. Intuitively, this simpler structure is the reason for the differences in computational hardness between the two problems. Nevertheless, the richer solution space of $\intervalsdxu$ provides non-trivial certifications to a \emph{strictly} larger class of inputs.

For uniform intervals we can define $\bot_{\mathbb{B}}, \top_{\mathbb{B}}$-operations for $\mathbb{B}(d)\vert_{\mathbf{x}}^{\mathbb{F}}$, analogous to Sec. \ref{sec:approximations}. To that end, consider a compact decision surface $\mathcal{D}_c \subseteq \mathbb{F}$ and an input $\mathbf{x} \in \mathcal{D}_c$. Let $\underline{\rho} > 0$ be the \emph{smaller} distance between $\mathbf{x}$ and a counterexample $\mathbf{x}^\prime \notin \mathcal{D}_c$, and $\overline{\rho} > 0$ the \emph{greatest} distance between $\mathbf{x}$ and another point $\mathbf{x}^{\prime\prime} \in \mathcal{D}_c$. Namely,
\begin{equation}
    \label{eq:cyclic-radii}
    \underline{\rho} = \inf \{\|\mathbf{x} - \mathbf{x}^\prime\|_\infty \mid \mathbf{x}^\prime \notin \mathcal{D}_c\}, ~~\overline{\rho} = \sup \{\|\mathbf{x} - \mathbf{x}^{\prime\prime}\|_\infty \mid \mathbf{x}^{\prime\prime} \in \mathcal{D}_c\}
\end{equation}
We define $\mathcal{D}_c\bot_{\mathbb{B}}\mathbf{x} = [\mathbf{x} - \underline{\rho}\mathbf{1}, \mathbf{x} + \underline{\rho}\mathbf{1}]$ and $\mathcal{D}_c\top_{\mathbb{B}}\mathbf{x} = [\mathbf{x} - \overline{\rho}\mathbf{1}, \mathbf{x} + \overline{\rho}\mathbf{1}]$. The following (set-theoretic) inequality relates the $\bot, \top$-operators in uniform and general intervals. For disambiguation, here we denote these operations with $\bot_{\mathbb{I}}, \top_{\mathbb{I}}$ for general intervals.
\begin{equation}
    \label{eq:bot-top-inequality}
    \mathcal{D}_c\bot_{\mathbb{B}}\mathbf{x} \subseteq \mathcal{D}_c \bot_{\mathbb{I}} \mathbf{x} \subseteq \mathcal{D}_c \subseteq \mathcal{D}_c \top_{\mathbb{I}} \mathbf{x} \subseteq \mathcal{D}\top_{\mathbb{B}}\mathbf{x}
\end{equation}
Noteworthy, for uniform intervals we have $\mathcal{D}_c\bot_{\mathbb{B}}\mathbf{x} = [\mathcal{D}_c\bot_{\mathbb{B}}\mathbf{x}]$. Since an objective $\metric(\cdot)$ respects set inclusion, $\metric(\mathcal{D}_c\bot_{\mathbb{B}}\mathbf{x}) \leq \metric(\mathcal{D}_c \bot_{\mathbb{I}} \mathbf{x}) \leq \metric(\mathcal{D}_c) \leq \metric(\mathcal{D}_c \top_{\mathbb{I}} \mathbf{x}) \leq \metric(\mathcal{D}\top_{\mathbb{B}}\mathbf{x})$. In Sec. \ref{sec:implementation}, we implement the $\bot, \top$-operators both for uniform and general interval certifications. We observe that the minimum edge length of a general interval certification is, on average, about twice greater than the same objective in uniform intervals.

    \section{Implementation}
    \label{sec:implementation}
    Here, we examine the implementation and practical implications of the theoretical framework discussed earlier. We developed the open source system \texttt{Parallele-\ pipedoNN}
. This system implements the operators discussed in Sec. \ref{sec:intervals}, \ref{sec:approximations}, and Appendix \ref{app:dichotomic}. In particular, Algorithm \ref{algo:generic-traversal} with \textsc{BUS} policy implements the $\top$-operator, while the $\bot$-operator is implemented in the Algorithm \ref{algo:generic-traversal} with \textsc{TDS} policy method. Recall that \textsc{BUS} policy is obtained by applying the parameters $I_0 = [\mathbf{x}, \mathbf{x}]$, $\varphi_{I, \mathcal{N}} = \mathscr{C}_{I, N}$, and $\square = \sqcup$. 
On the other hand, \textsc{TDS} is obtained for the parameters $I_0 = \mathbb{F}$, $\varphi_{I, \mathcal{N}} = \mathscr{S}_{I, N}$, and $\square = /$. Moreover, since Algorithm \ref{algo:non-det-expand} is non-deterministic, we implemented a deterministic \textsc{Sequential Dichotomic Expansion (SDE)} method, presented in Appendix \ref{app:dichotomic}. \textsc{SDE} ensures maximality. Additionally, we include in our analysis evaluations regarding Algorithm \ref{algo:generic-traversal} with either \textsc{BUS} or \textsc{TDS} policies for the case of uniform intervals. We denote these special cases of the algorithms as $\mathbb{B}$-\textsc{BUS} and $\mathbb{B}$-\textsc{TDS}, respectively. 

In Tbl. \ref{tab:algorithms}, we overview the presented algorithms. For minimal complete certifications, we presented two algorithms, \textsc{BUS} and $\mathbb{B}$-\textsc{BUS}, for general and uniform intervals, respectively. Recall that the minimal complete uniform interval certification is \emph{not} complete. \textsc{SDE} computes maximally sound interval certifications. \textsc{TDS} ensures voluminosity, under the assumption of Prop. \ref{prop:non-triviality-lb}. Combining these two practices, we have a voluminous and maximally sound certification. Finally, $\mathbb{B}$-\textsc{TDS} computes maximally complete certifications for uniform intervals. Naturally, $\mathbb{B}$-\textsc{TDS} also ensures voluminosity, under the same conditions of Prop. \ref{prop:non-triviality-lb}. For complete certifications, voluminosity is irrelevant. A complete interval will always have volume if $\mathcal{D}_c$ has volume. Our experiments provide counterexamples for voluminosity in \textsc{SDE}.

\begin{table}[t]
    \centering
    \begin{adjustbox}{width=\textwidth}
    \begin{tabular}{l c c c c c c}
        \toprule\toprule
         Algorithm~ & ~Operator~ & ~\thead{\textbf{S}ound/\\ \textbf{C}omplete}~ & ~\thead{Min./\\Max.}~ & ~Vol.~ & ~Complexity~ & ~Theory  \\
         \midrule\midrule
         \textsc{BUS} & $\mathcal{D}_c\top_{\mathbb{I}}\mathbf{x}$ & C & Min. & -- & $O[d\cdot\mathcal{A}(\universe)/\delta]$ & Th. \ref{theo:algo-min-complete}\\
         \textsc{TDS} & $\mathcal{D}_c\bot_{\mathbb{I}}\mathbf{x}$ & S & -- & (\checkmark) & $O[d\cdot\mathcal{A}(\universe)/\delta]$ & Th. \ref{theo:algo-max-sound} \\
         \textsc{SDE} & $[I]$ & S & Max. & \xmark & $O[d\cdot\log(\mathcal{A}(\universe))]$ & Prop. \ref{prop:dichotomic}\\
         TDS+SDE & $[\mathcal{D}_c\bot_{\mathbb{I}}\mathbf{x}]$ & S & Max. & (\checkmark) & $O[d\cdot\mathcal{A}(\universe)/\delta]$ & Th. \ref{theo:algo-max-sound}, Prop. \ref{prop:dichotomic}\\
         \midrule
         $\mathbb{B}$--BUS & $\mathcal{D}_c\top_{\mathbb{B}}\mathbf{x}$ & C & Min. & -- & $O[\log(\mathcal{A}(\universe))]$ & Th. \ref{theo:sound-max-cyclic}\\
         $\mathbb{B}$--TDS & $\mathcal{D}_c\bot_{\mathbb{B}}\mathbf{x}$ & S & Max. & (\checkmark) & $O[\log(\mathcal{A}(\universe))]$ & Th. \ref{theo:sound-max-cyclic}\\
         \bottomrule\bottomrule
    \end{tabular}
    \end{adjustbox}
    \caption{Algorithms \& Properties. With ``(\checkmark)'' we denote a property that is proven to hold (under the assumption $\mathbf{x} \in \mathcal{D}_c^\circ$). With ``\xmark'' we denote a disproved property. With ``--'' we denote irrelevant properties. In the rightmost column, we refer to the theoretical result (theorem or proposition) proving the algorithm's properties. In $\mathbb{B}$-\textsc{BUS}, we support our claim using the \emph{proof} of Th. \ref{theo:sound-max-cyclic}. The complexity/correctness proofs of $\mathbb{B}$-\textsc{BUS} and $\mathbb{B}$-\textsc{TDS} are practically identical. For $\mathbb{B}$-\textsc{BUS}, $\mathbb{B}$-\textsc{TDS} maximality/minimality is examined for uniform intervals.}
    \label{tab:algorithms}
\end{table}

\subsection{Experimental Evaluation}

Below we describe our experimental setup.

\begin{description}
    \item[Hardware.] The experiments were performed in \emph{parallel} on an Ubuntu 18.04 machine, with Intel Xeon E5-2640 v4 CPU, at 2.394GHz, with 38 cores, with 128GB RAM. The experiments ran in parallel, utilizing 35 cores.
    
    \item[Software Dependencies.] Our software is written in Python v3.8.16. The soundness and completeness of the oracles of eq. \eqref{eq:sound-oracle} and eq. \eqref{eq:complete-oracle} respectively, are implemented using the Marabou v2.0%
    \footnote{
        See \url{https://github.com/NeuralNetworkVerification/Marabou} and \cite{marabou-2}.
    } NN verifier. Our implementation takes as input a multilayered perceptron in open neural network exchange (ONNX) v1.16.0\footnote{
        See \url{https://github.com/onnx/onnx}.
    } format. For linear algebra computations, we used the NumPy v1.23.5\footnote{
        See \url{https://github.com/numpy/numpy}.
    } library. For visualization, we used the Matplotlib v3.7.2\footnote{
        See \url{https://github.com/matplotlib/matplotlib}.
    } library. The MLPs used in the experiments were trained from scratch, using TensorFlow v2.12.0\footnote{
        See \url{https://github.com/tensorflow/tensorflow}.
    }.

    \item[Algorithms' Parameters.] We evaluated all the algorithms of Tbl. \ref{tab:algorithms} using the same parameters. We set the precision constant $\delta$ to $0.1$. We also set a \emph{timeout} variable to 1 hour. For TDS+SDE the timeout is 2 hours, 1 hour for each component. The maximum number of iterations was set to 10,000.

    \item[Training Dataset.] We use 2 datasets, namely MNIST \cite{mnist} and Fashion MNIST \cite{fashion-mnist}. Both datasets consist of 28$\times$28, grayscale images, belonging to 10 classes. However, Fashion MNIST images are \emph{significantly} more complex than MNIST.
    
    \item[Neural Networks.] We consider the following MLP architecture. applied our algorithms to 2 MLPs, of the same architecture, trained on the MNIST and Fashion MNIST datasets, respectively. Including the input and output layers, we have the architecture $\angles{784, 32, 10, 10}$, w.r.t. Def. \ref{def:mlp}. This corresponds to 25,450 trainable parameters. For training, we used the Adam \cite{adam} algorithm, Glorot \cite{glorot} weight initialization, and the Categorical Crossentropy loss. By training on the 2 datasets above, this results in two MLPs, achieving $94\%$ and $82\%$ test-set accuracy for the MNIST and Fashion MNIST, respectively.
    \item[Inputs.] For each MLP, we randomly choose 5 images of the 10 classes of the test set (a total of 50 images per MLP).
\end{description}

\begin{table}[t]
    \centering
    \begin{adjustbox}{width=\textwidth}
    \begin{tabular}{l l c c c c}
        \toprule\toprule
        \multirow{2}{*}{\thead{\\Dataset}}~~ & \multirow{2}{*}{\thead{\\Algorithm}}~~~~ & \multicolumn{4}{c}{Average}\\ \cline{3-6}
        &  & \thead{Time\\ \textbf{s}ec./\textbf{m}in.}~~ & \thead{\# Verif. Calls}~~ & \thead{Min. Edge\\ Len. $\alpha(\cdot)$}~~ & \thead{Timeouts}
        \\ \midrule\midrule
        \multirow{6}{*}{MNIST} & \textsc{BUS} & $38.54$m & 2085.73 & 0.99 & $0$\\
        & \textsc{TDS} & $51.13$m & 3753.02 & 0.13 & $28$\\ 
        & \textsc{SDE} & $32.78$m & 2963.86 & 0.0 & $13$\\
        & TDS+SDE & $+42.04$m & +1030.55 & 0.13 & 25\\
        \cline{2-6}
        & $\mathbb{B}$--BUS & $3.92$s & 4 & 0.94 & $0$\\
        & $\mathbb{B}$--TDS & $21.06$s & 4 & 0.07 & $0$\\
        \midrule\midrule
        \multirow{6}{*}{\thead{Fashion\\ MNIST}} & \textsc{BUS} & 22.92m & 1235.26 & 0.1 & 0\\
        & \textsc{TDS} & 50.17m & 3272.76 & 0.18 & 32\\ 
        & \textsc{SDE} & 22.82m & 3469.48 & 0.0 & 3\\
        & TDS+SDE & +23.91m & +1374.45 & 0.18 & 1\\
        \cline{2-6}
        & $\mathbb{B}$--BUS & 3.5s & 4 & 0.94 & 0\\
        & $\mathbb{B}$--TDS & 5.15s & 4 & 0.12 & 0 \\
        \bottomrule\bottomrule
    \end{tabular}
    \end{adjustbox}
    \caption{Experimental evaluation on the MNIST and Fashion MNIST datasets.}
    \label{tab:experiments}
\end{table}

Tbl.~\ref{tab:experiments} reports descriptive statistics on the CPU time, the number of verification oracle calls, and the interval's minimum edge length $\alpha(\cdot)$. We evaluate all six algorithms presented in this paper on both datasets. Overall, the empirical results align with our theoretical analysis. In particular, $\mathbb{B}$-\textsc{BUS} and $\mathbb{B}$-\textsc{TDS} are approximately an order of magnitude faster than their general-intervals counterpart, as expected, since they do \emph{not} scale with the input dimension. This speed-up, however, comes at a cost, namely, $\mathbb{B}$-\textsc{TDS} attains only about half of the minimum edge length achieved by \textsc{TDS} (or roughly two-thirds on Fashion-MNIST). Finally, we observe that the Fashion MNIST MLP appears, on average, more robust that the MNIST MLP, despite achieving lower classification accuracy. This is consistent with prior findings, obtained via different methodologies, on the robustness-accuracy trade-off \cite{madry-2018,zhang-2019,tsipras-2019,raghunathan-2020}. More detailed statistical results are provided in Appendix \ref{app:experiments}.


    \section{Conclusions \& Future Work}
    \label{sec:conclusions}
    This work develops a framework for computing maximally sound and minimally complete interval certifications for MLPs.
In Sec.~\ref{sec:approximations}, we explore the interval certifications lattice, defining the $\top,\bot,[\cdot]$ operators. We develop algorithms that guarantee minimum completeness (Th.~\ref{theo:algo-min-complete}) and maximal soundness (Th.~\ref{theo:algo-max-sound}, \ref{theo:non-det-expand}). In Sec.~\ref{sec:optimization}, we study optimization on interval certifications. We observe intriguing asymmetries. The minimum complete certification can be computed in polynomial oracle calls. However, we extend previous results, by showing a stronger intractability result for sound maximization (Th.~\ref{theo:sound-max-hard}). Nevertheless, when optimization problems, are restricted to uniform intervals ($\ell_\infty$--spheres) become solvable in \emph{logarithmic} number of oracle calls. Finally, we implement our theoretical insights in the \texttt{ParallelelpipedoNN} system, (Sec.~\ref{sec:implementation}), which we evaluate on MNIST and Fashion MNIST. As future work, we plan to extend our analysis to more general \emph{polyhedral} certifications.
    \bibliographystyle{splncs04}
    \bibliography{formal-nn-verif}

     \appendix

    \section{Omitted proofs}
    \label{app:proofs}
    \subsection{Proofs of Section \ref{sec:intervals}}

\numericalgeometricmean*

\begin{proof}
    Applying straightforward computations we take,
    \begin{align*}
            \mathcal{A}(I) & = \|b_i - \ell_i\|_{\infty} = \frac{1}{d}\sum_{i \in [d]}\max_{i \in [d]}\{u_i - \ell_i\} &&& \\
            & \geq \frac{1}{d}\sum_{i \in [d]} u_i - \ell_i && \left[ \text{ equals } \frac{1}{d} \pi(I)\right]  \\
            & \geq \sqrt[d]{\prod_{i \in [d]} u_i - \ell_i} && \left[\begin{array}{l}
            \text{by arithmetic-geometric} \\
            \text{mean inequality. Equals } \sqrt[d]{v(I)}
            \end{array}
            \right]\\
            & \geq \sqrt[d]{\left(\min_{i \in [d]} \{u_i - \ell_i\}\right)^d} && \left[ \text{ equals } \alpha(I))\right].
    \end{align*}

    $\hfill\Box$
\end{proof}

\intervalexclusion*

\begin{proof}
    Consider any interval $J = [\bell', \bu'] \subseteq I$ such that $\mathbf{x}' \notin J$. Since $\mathbf{x}' \notin J$, there exists at least one coordinate $i \in [d]$ such that either $u'_i < x'_i$ or $\ell'_i > x'_i$. In either case, the length of $J$ along coordinate $i$ satisfies $u'_i - \ell'_i \leq \max\{x'_i - \ell_i, \ u_i - x'_i\}$.

    Let $k \in [d]$ be the coordinate selected in Definition \ref{def:interval-exclusion}, that is $k = \arg\max_{i \in [d]} |x_i - x'_i|$. By construction, the interval $I/\mathbf{x}'$ is obtained by shrinking $I$ only along coordinate $k$. Therefore, we have
    \[
        \alpha(I/\mathbf{x}') = \min \left\{ \min_{i \neq k}(u_i - \ell_i), \ \max\{x'_k - \ell_k, \ u_k - x'_k\} \right\}.
    \]
    Since $J \subseteq I$, we have $u'_i - \ell'_i \le u_i - \ell_i$ for all $i \in [d]$, and since $J$ excludes $\mathbf{x}'$, at least one coordinate must satisfy
    \[
        u'_i - \ell'_i \leq \max_{i \in [d]}\{x'_i - \ell_k, \ u_k - x'_i\} \leq \max\{x'_k - \ell_k, \ u_k - x'_k\}.
    \]
    Therefore, we have $\alpha(J) = \min_{i \in [d]}\{u'_i - \ell'_i\} \leq \alpha(I/\mathbf{x}')$, and the proof is complete.\\

    \hfill$\Box$
\end{proof}

\subsection{Proofs of Section \ref{sec:approximations}}

\completecertification*

\begin{proof}
    Let $\mathbf{y} \not\in I$. Then there exists a coordination $i$ such that either $y_i < x_i - \underline{R}_i$ or $y_i > x_i + \overline{R}_i$. By formulas \eqref{def:top-operator}, no point of $\mathcal{D}_c$ extends further in that direction. Hence, $\mathbf{y} \not\in \mathcal{D}_c$, and therefore $\kappa(\mathbf{y}) \neq c$. Since exiting the interval necessarily the prediction changes, we establish completeness.

    Let $J  = [\bell, \bu]$ be any complete certification containing $\mathbf{x}$, implying that for each coordinate $i$, we have
    \[
        \ell_i \leq \min \{x'_i \mid \mathbf{x}' \in \mathcal{D}_c\}, \quad u_i \geq \max \{x'_i \mid \mathbf{x}' \in \mathcal{D}_c\}.
    \]
    Thus, it holds $\mathcal{D}_c \top \mathbf{x} \subseteq J$. If the inclusion were strict, then there would exist a point outside $\mathcal{D}_c$ still inside $J$, violating completeness. Therefore, we conclude that $J = \mathcal{D}_c \top \mathbf{x}$.

    Let $\mathbf{x}_1, \mathbf{x}_2 \in \mathcal{D}_c$. Observe that the radii $\underline{\mathbf{R}}$ and $\overline{\mathbf{R}}$ depend only on the extremal coordinates of $\mathcal{D}_c$, not on the choice of $\mathbf{x}$. Changing the reference point $\mathbf{x}$ shifts both endpoints equally, leaving the interval invariant. Hence, it holds $\mathcal{D}_c\top\mathbf{x}_1 = \mathcal{D}_c\top\mathbf{x}_2$, and the proof is complete.

    $\hfill\Box$
\end{proof}

\algomincomplete*

\begin{proof}
    Let $I_j = [-\bell^j, \bu^j]$, $\bell^j, \bu^j \geq \mathbf{0}$, $j \in \mathbb{N}$, be the interval in the $j$-th iteration of the Algorithm \ref{algo:generic-traversal}, when we apply the property $\mathscr{C}_{\mathcal{N}, I}(\mathbf{x}^\prime)$ and the operator $\sqcup$. Consider some $k \in \mathbb{N}$, s.t. $I_k = I_{k+1}$. W.l.o.g., let $k$ be the smallest natural with the latter property. For $I_k$ holds that $I_k \supseteq \mathcal{D}$, otherwise the $\mathscr{C}_{\mathcal{N}, I}(\mathbf{x}^\prime)$ oracle would find a counterexample $\mathbf{x}^a$, and $I_k \neq I_k \sqcup \mathbf{x}^a$, for each $\mathbf{x}^a \in I_k$. Under this configuration, the Algorithm \ref{algo:generic-traversal} terminates after a finite number of steps, with $I_k \supseteq \mathcal{D}$. Moreover, since $\mathcal{D}_c\top\mathbf{x}$ is the unique complete certification it holds $I_k \supseteq \mathcal{D}_c\top\mathbf{x}$.

    Moreover, for each $j \in [k]$, let $\mathbf{x}^j$ be the $j$-th counterexample returned from the $\mathscr{C}_{I, \mathcal{N}}(\mathbf{x}^\prime)$ oracle, s.t. $I_{j - 1} \sqcup \mathbf{x}^j = I_j$, and $I_0 = [\mathbf{0}, \mathbf{0}]$. For each $j \in [k]$, we have a strictly increasing sequence of intervals $I_0 \subset I_1 \subset \cdots \subset I_k$, since $I_{j - 1} \subset I_{j - 1} \sqcup \mathbf{x}^j = I_j$. Consider the set of the first $j$ witnesses returned by the oracle, $\mathcal{X}_j = \{\mathbf{x}^i \mid i \in [j]\}$. Naturally, $\mathcal{X}_1 \subset \mathcal{X}_2 \subset \cdots \subset \mathcal{X}_k$. For the $j$-th iteration, and the $i$-th coordinate, consider the values
    \begin{equation}\label{eq:algo-min-complete}
        \underline{\theta}^j_i = \max \{-x^a_i \mid \mathbf{x}^a \in \mathcal{X}_j\}, \quad \overline{\theta}^j_i = \max \{x^a_i \mid \mathbf{x}^a \in \mathcal{X}_j\}.
    \end{equation}
    Clearly, it holds $\underline{\theta}^j_i, \overline{\theta}^j_i \geq 0$. Since, $\mathcal{X}_j \subset \mathcal{D}$, for each $j \in [k]$, there exist at least one index $i \in [d]$, s.t. $\underline{R}_i \leq \underline{\theta}^j_i $ and $\overline{\theta}^j_i \leq \overline{R}_i$. Next, for the $j$-th iteration, and for every $i \in [d]$, the $\sqcup$ implies that $\ell^j_i = \underline{\theta}^j_i + \delta$ and $u^j_i = \overline{\theta}^j_i + \delta$. Therefore, for each $j \in [k]$ we have $\ell^j_i \geq \underline{R}_i + \delta$ and $u^j_i = \overline{R}_i + \delta$, concluding that $I_k \subseteq \mathcal{D}_c\top\mathbf{x} + [-\delta \mathbf{1}, \delta \mathbf{1}]$.

    Finally, we show the complexity of the Algorithm \ref{algo:generic-traversal}, when we apply the property $\mathscr{C}_{I, \mathcal{N}}(\mathbf{x}^\prime)$ and the operator $\sqcup$. Consider the potential function $\Phi([-\bell, \bu]) = \sum_{i\in [d]} \ell_i + u_i$, for $\bell, \bu \geq \mathbf{0}$. For each $j \in [k]$, the coordinate $i$, s.t. $\ell^j_i \leq \ell^{j - 1}_i + \delta$ or $u^j_i \leq u^{j - 1}_i + \delta$. Thus, $\Phi(I_j) \leq \Phi(I_{j-1}) + \delta$, and $\Phi(I_0) = \Phi([\mathbf{0}, \mathbf{0}]) = 0$. Now, from Prop. \ref{prop:numerical-geometric-mean} it holds, $\Phi(I_k) = \sum_{i\in [d]} \ell^k_i + u^k_i \leq d \cdot \mathcal{A}(I_k)$. Hence, we take $\Phi(I_k) \leq d \cdot \mathcal{A}(\mathcal{D}_c\top\mathbf{x} + [-\delta \mathbf{1}, \delta \mathbf{1}]) \leq d \cdot \left(\mathcal{A}(\mathcal{D}_c\top\mathbf{x}) + 2\delta \right)$. So, we have $T = [\Phi(I_k) - \Phi(I_0)]/\delta = O(d \cdot \left(\mathcal{A}(\mathcal{D}_c\top\mathbf{x}) + 2\delta - 0 \right)/\delta) = O(d \cdot \left(\mathcal{A}(\mathcal{D}_c\top\mathbf{x})\right)/\delta)$.

    $\hfill\Box$
\end{proof}

\propmaximalsoundapprox*

\begin{proof}
    For sake of contradiction, let $\mathbf{x}^a \in \mathcal{D}\bot\mathbf{0} \setminus \mathcal{D}$. Since $\mathbf{x}^a \in \mathcal{D}\bot\mathbf{0}$, for each $i \in [d]$, it holds that ~$-\uradi \leq x^a_i \leq \oradi$. Consider some $j \in [d]$, s.t. $\mathbf{x}^a \in \underline{V}_j$ or $\mathbf{x}^a \in \overline{V}_j$~ (but not both, since $\mathbf{x}^a \neq \mathbf{0}$). Let $\mathbf{x}^a \in \underline{V}_j$. Then, $\uradj_j < -x^a_j$, or $-\uradj_j > x^a_j$. A contradiction. When $\mathbf{x}^a \in \overline{V}_j$, we work similarly.

    $\hfill\Box$
\end{proof}

\theoalgomaxsound*

\begin{proof}
    Let $I_j = [-\bell^j, \bu^j]$, $\bell^j, \bu^j \geq \mathbf{0}$, $j \in \mathbb{N}$, be the interval in the $j$-th iteration of the Algorithm \ref{algo:generic-traversal}, when we apply the property $\mathscr{S}_{\mathcal{N}, I}(\mathbf{x}^\prime)$ and the operator $\sqcap$. W.l.o.g., let $k$ be the smallest natural with the latter property. For $I_k$ holds that $I_k \subseteq \mathcal{D}$. Otherwise, the soundness oracle $\mathscr{S}_{\mathcal{N}, I}(\mathbf{x}^\prime)$ would find a counterexample $\mathbf{x}^a$. And $I_k \neq I_k \sqcap\mathbf{x}^a$, for each $\mathbf{x}^a \in I_k$. Therefore, Algorithm \ref{algo:generic-traversal} terminates after a finite number of steps, with $I_k \subseteq \mathcal{D}$.

    Moreover, for each $j \in [k]$, let $\mathbf{x}^j$ be the $j$-th counterexample returned by the $\mathscr{S}_{I, \mathcal{N}}(\mathbf{x}^\prime)$ oracle, s.t. $I_{j-1} \sqcap \mathbf{x}^j = I_{j}$, and $I_0 = \universe$. For each $j \in [k]$, we have a strictly decreasing sequence of intervals $I_0 \supset I_1 \supset \cdots \supset I_k$, since $I_{j-1} \supset I_{j-1} \sqcap \mathbf{x}^j = I_{j}$. Consider the set of the first $j$ witnesses returned by the oracle, $\mathcal{X}_j = \{\mathbf{x}^i \mid i \in [j]\}$. Naturally, $\mathcal{X}_1 \subset X_2 \subset \cdots \subset \mathcal{X}_k$. For the $j$--th iteration, and the $i$--th coordinate, consider the values 
    \begin{equation}
        \label{eq:algo-max-sound}
        \underline{\tau}^j_i = \min\{x^a_i \mid \mathbf{x}^a \in \underline{V}_i \cap \mathcal{X}_j\}, \quad
        \overline{\tau}^j_i = \min\{x^a_i \mid \mathbf{x}^a \in \overline{V}_i \cap \mathcal{X}_j\}
    \end{equation}
    Now, applying the operator $\sqcap$ to intervals $[-\bell, \bu]$ and $[\mathbf{x}^a, \mathbf{x}^a]$, for any $\mathbf{x}^a \in \mathbb{R}^d$, we have
    \begin{equation}
        \label{eq:specification-op}
        \begin{cases}
            \ell^\prime_i = x^a_i - \delta, & \text{if}~~ \mathbf{x}^a \in \underline{V}_i \cap [-\bell, \bu],\\
            u^\prime_i = x^a_i - \delta, & \text{if}~~ \mathbf{x}^a \in \overline{V}_i \cap [-\bell, \bu].
        \end{cases}
    \end{equation}
    Otherwise, $[-\bell, \bu] =[-\bell, \bu] \sqcap \mathbf{x}^a$, when $\mathbf{x}^a \notin [-\bell, \bu]$. If $\mathbf{x}^a \in \underline{V}_i$, we update the $i$-th value of $\bell$ to be $x^a_i - \delta$. Symmetrically, if $\mathbf{x}^a \in \overline{V}_i$, we update the $i$-th value of $\bu$ to be $x^a_i - \delta$.

    Since, $\mathcal{X}_j \subset \universe \setminus \mathcal{D}$, we have $\uradi \leq \underline{\tau}_i^j$ and $\oradi \leq \overline{\tau}_i^j$, for each $j \in [k]$. From eq. \eqref{eq:specification-op}, for the $j$-th iteration, and the $i$-th coordinate, we have $\ell^j_i = \underline{\tau}^j_i - \delta$ and $u^j_i = \overline{\tau}^j_i - \delta$. Thus, for each $j \in [k]$ we have $\ell^j_i \geq \uradi - \delta$ and $u^j_i \geq \oradi - \delta$. Therefore, $I_k \supseteq \mathcal{D}\bot\mathbf{x} - \delta\mathbf{1}$.

    Finally, we show the \textsc{GenericTraversal} algorithm's complexity. Consider the potential function $\Phi([-\bell, \bu]) = \sum_{i \in [d]} \ell_i + u_i$, for $\bell, \bu \geq \mathbf{0}$. For each $j \in [k]$, there is a coordinate $i$, s.t. $\ell^j_i \leq \ell^{j-1}_i - \delta$ or $u^j_i \leq u^{j-1}_i - \delta$. Thus, $\Phi(I_j) \leq \Phi(I_{j-1}) - \delta$. Additionally, $\Phi(I_0) = d\cdot\mathcal{A}(\universe)$, and $\Phi(I_k) \geq d\cdot\mathcal{A}(\mathcal{D}\bot\mathbf{x}) - \delta$. Hence, for the complexity of the Algorithm \ref{algo:generic-traversal}, we have $T = [\Phi(I_0) - \Phi(I_k)]/\delta = O(d[\mathcal{A}(\universe) - \mathcal{A}((\mathcal{D}\bot\mathbf{x}) + \delta]/\delta) = O(d[\mathcal{A}(\universe) - \mathcal{A}(\mathcal{D}\bot\mathbf{x})]/\delta)$.

    $\hfill\Box$
\end{proof}

\nontrivialitylb*

\begin{proof}
    If $\mathbf{x} \in \mathcal{D}^\circ_c$, then there is some $\rho > 0$, s.t. $\infball(\mathbf{x}, \rho) \subseteq \mathcal{D}$. Choose $\rho$ to be the greatest real number s.t. the previous statement holds. Then, for each $i \in [d]$ holds $\uradi, \oradi \geq \rho > 0$. Thus, $\alpha(\mathcal{D}_c\bot\mathbf{x}) > 0$.

    $\hfill\Box$
\end{proof}

\lemmaximalsoundness*

\begin{proof}
    $[\Rightarrow]$ We prove the contrapositive. Let some $i \in [d]$ and some $\delta > 0$, s.t. either $[-(\bell + \delta\mathbf{e}^i), \bu] \setminus \mathcal{D}_c = \varnothing$ \emph{or} $[-\bell, \bu + \delta\mathbf{e}^i] \setminus \mathcal{D}_c = \varnothing$ (or both). W.l.o.g. let $[-(\bell + \delta\mathbf{e}^i), \bu] \setminus \mathcal{D}_c = \varnothing$. Since $[-(\bell + \delta\mathbf{e}^i), \bu] \supset [-\bell, \bu]$, then $[-\bell, \bu]$ is not maximal.

    $[\Leftarrow]$ Let for every $i \in [d]$, and every $\delta > 0$, we assume that holds, $[-(\bell + \delta\mathbf{e}^i), \bu] \setminus \mathcal{D}_c \neq \varnothing$ \emph{and} $[-\bell, \bu + \delta\mathbf{e}^i] \setminus \mathcal{D}_c \neq \varnothing$. Now assume some $I^\prime \supset I = [-\bell, \bu]$. There are $i \in [d]$ and $\delta > 0$, s.t. $I^\prime \supseteq [-(\bell + \delta\mathbf{e}^i), \bu]$ or $I^\prime \supseteq [-\bell, \bu + \delta\mathbf{e}^i]$ (or both). Therefore, $I^\prime \setminus \mathcal{D}_c \neq \varnothing$. Hence, $I$ is maximal.

    $\hfill\Box$
\end{proof}

\nondetexpand*
\begin{proof}
    We start from the trivial sound certification $X = [\mathbf{x}, \mathbf{x}] \in \intervalsdxu$. Let, $I^\ast = [\bell^\ast, \bu^\ast] \in [X]$ be an arbitrary maximally sound interval. Observe that Algorithm \ref{algo:non-det-expand}, at each iteration, expands exactly one coordinate of either the lower or upper endpoint by $\delta$. Therefore, it performs monotone expansions of the current interval $I$. Thus, the algorithm defines a monotone path in the interval lattice $\angles{\intervalsdxu, \subseteq}$, starting from $X$.

    Next, we consider the following non-deterministic strategy: i) if the algorithm expands a coordinate $k$ of the lower endpoint, then it chooses $k$ s.t. $\ell_k < \ell^\ast_k$, and ii) if the algorithm expands a coordinate $k$ of the upper endpoint, then it chooses $k$ s.t. $u_k < u^\ast_k$. As long as the current interval $I$ satisfies $I \subseteq I^\ast$, all expansions are sound. Since each coordinate is expanded in increments of size $\delta$, after at most
    \[
        \left\lceil \frac{u^\ast - x_k}{\delta} \right\rceil \quad \text{ or } \quad \left\lceil \frac{x_k - \ell^\ast}{\delta} \right\rceil
    \]
    steps per coordinate, the algorithm reaches an interval $J$ s.t. $I - [-\delta\mathbf{1}, \delta\mathbf{1}] \subseteq J \subseteq I$. Observe that, by Lem. \ref{lem:maximal-soundness}, for every coordinate $i$ and every $\delta > 0$, any further unilateral expansion beyond $I^\ast$ results in a counterexample.

    Finally, for the complexity, there are $2d$ coordinates of all endpoints, and they can be expanded at most $O(1/\delta)$ times. Each expansion requires one call in the soundness oracle $\mathscr{S}_{I, \mathcal{N}}(\mathbf{x}^\prime)$, costing $\mathcal{A}(\mathbb{F})$. Therefore, the total number of non-deterministic steps is
    \[
        O \left( d \cdot \frac{\mathcal{A}(\mathbb{F})}{\delta}\right),
    \]
    and the proof is complete.

    $\hfill\Box$
\end{proof}

\subsection{Proofs of Section \ref{sec:optimization}}

\theosoundmaxhard*

\begin{proof}
    The proof relies on a reduction from q-\textsc{MER} to \textsc{SIM}. Given an instance $\mathcal{I}$ in q-\textsc{MER}, we can construct an instance $\mathcal{I}'$ in \textsc{SIM} as follows. Let $\mathbf{x} := \mathbf{q}$, we define the decision surface $\mathcal{D}_c := \universe \setminus  \mathcal{F}$. Since the set $\mathcal{F}$ is finite and $\universe$ is compact, the set $\mathcal{D}_c$ is compact as well. Next, we interpret every forbidden point $\mathbf{f} \in \mathcal{F}$ as a counterexample, that is, given the classifier $\kappa(\cdot)$, it holds $\kappa(\mathbf{f}) \neq c$. Using the terminology of the soundness oracle, eq. \eqref{eq:sound-oracle}, we have that $\mathbf{f} \models \mathscr{I}_{I, \mathcal{N}}(\mathbf{f})$ for all points in $\mathcal{F}$. Contrarily, sound intervals are precisely those that exclude $\mathcal{F}$. Therefore, an interval $I \in \intervalsdxu$, with $I \subseteq \mathcal{D}_c$, is sound iff $I \cap \mathcal{F} = \emptyset$. Since we have set $\mathbf{x} := \mathbf{q}$, and the objective constraint $v(I) \geq \gamma$, for $\gamma > 0 $, is identical in both problems, the reduction is established. Hence, the $\mathcal{I}$ is feasible iff the $\mathcal{I}'$ is feasible. Since the problem q-\textsc{MER} can not be decided in polynomial time, the problem \textsc{SIM} can not be decided in polynomial time, as well, and the proof is complete.

    $\hfill\Box$
\end{proof}

\lemsoundmaxhard*

\begin{proof}
    Let $G = (V, E)$ a simple undirected graph with $|V| = d$. We fix any constant $w > 0$ and wlog we work on $[0, 1]^d$. From Lem. 2 in \cite{empty-boxes} if $I \subseteq [0, 1]^d$ is a maximum-volume feasible interval, then $\partial I$ contains $\mathbf{0}$ and the $i$-th dimension of $I$ is either $w$ or $1$. Therefore, the maximum-volume interval $I$ has $\ell_i = 0$ for every $i$, and each $u_i \in \{w, 1\}$, meaning that $I = [\mathbf{0}, \mathbf{u}]$. The interval $I$ has volume $v(I) = \prod_{i \in [d]} (u_i - \ell_i)$. Applying Th. 2 from \cite{empty-boxes}, graph $G$ has an independent set of size at least $k$ iff there exists an interval $I \in \universe$ of volume at least $w^k$. Consequently, deciding whether there exists any empty interval of volume at least $w^k$ is \textbf{NP}-hard. Taking $\gamma = w^k$ and $\mathbf{q} = \mathbf{0}$, every interval of the form $[\mathbf{0}, \mathbf{u}]$ contains $\mathbf{q}$, so deciding the existence of an empty rectangle that contains the query point $\mathbf{q}$ with $v(I) \geq \gamma$ is \textbf{NP}-hard. Therefore, the q-\textsc{MER} decision problem is \textbf{NP}-hard, and the proof is complete.

    $\hfill\Box$
\end{proof}

\soundmaxcyclic*

\begin{proof}
    We present an algorithm with executed in a logarithmic number of oracle calls in Algorithm \ref{algo:cyclic-dichotomic}. Essentially, we do a \emph{dichotomic} search in the real interval $[0, \mathcal{A}(\mathbb{F})]$.

    \begin{algorithm}[t]
    \DontPrintSemicolon
    \caption{$\mathbb{B}$--\textsc{TopDown Search}}
    \label{algo:cyclic-dichotomic}
    \KwInput{$\mathbf{x} \in \mathbb{F}$, the center of the uniform interval certification, and a percision constant $\delta > 0$.}
    \KwOutput{A maximally sound uniform interval certification}

    $\mathtt{low} \gets 0$; $\mathtt{high} \gets \mathcal{A}(\mathbb{F})$; 
    
    \While{$\mathtt{high} - \mathtt{low} > \delta$}{
        $\rho \gets (\mathtt{high} - \mathtt{low}) / 2$

        $B \gets \mathcal{B}(\mathbf{x}, \rho)$
        
        \uIf{$\exists \mathbf{x}^\prime$, s.t. $\mathbf{x}^\prime \models \mathscr{S}_{B, \mathcal{N}}$}{
            $\mathtt{high} \gets \rho$
        }
        \Else{
            $\mathtt{low} \gets \rho$
        }
    }
    
    \Return{$B$}
\end{algorithm}

    $\hfill\Box$
\end{proof}

\noindent Note that Algorithm \ref{algo:cyclic-dichotomic} can easily be modified to compute complete uniform certifications. We simply use the $\mathscr{C}_{B, \mathcal{N}}$ predicate, instead of $\mathscr{S}_{B, \mathcal{N}}$, and switching steps 6, 8.
%
%
%

    \section{Geometry of the $\mathbb{R}^d$ Cone Partition}
    \label{app:geometry}
    In this section, we briefly discuss the geometry involved in the cone decomposition of $\mathbb{R}^d$, introduced in eq. \eqref{eq:cones} of Sec. \ref{sec:approximations}. To that end, we firstly introduce some elementary concepts of affine geometry. Then we proceed to examine the intersections of these cones.

\subsection{Elements of Affine Geometry}

We begin with convex bodies. For $\bell, \bu \in \mathbb{R}^d$, the line segment with endpoints $\bell$ and $\bu$ is defined as $\text{L}(\bell, \bu) = \{\mathbf{x} \in \mathbb{R}^d \mid x = \bell + t\cdot(\bu - \bell), ~t\in [0, 1]\}$. A set $S \subseteq \mathbb{R}^d$ is said to be convex if $L(\bell, \bu) \subseteq S$ for every $\bell, \bu \in S$. A convex body is a convex and compact subset of $\mathbb{R}^d$ with nonempty interior. An important class of convex sets is the \emph{polyhedra}. Let $A \in \mathbb{R}^{m \times n}$ and $\bu \in \mathbb{R}^m$. The set $P = \{\mathbf{x} \in \mathbb{R}^n \mid A\mathbf{x} \leq \bu\}$ is called a (convex) polyhedron. If $\bu = \mathbf{0}$, we obtain a polyhedral cone $C = \{\mathbf{x} \in \mathbb{R}^m \mid A\mathbf{x} \leq \mathbf{0}\}$.

A set of vectors $\mathbf{v}_1, \dots, \mathbf{v}_n \in \mathbb{R}^d$ is \emph{affinely independent} if the vectors $\left[ \mathbf{v}_1 \atop 1 \right]$, $\dots$, $\left[ \mathbf{v}_n \atop 1 \right]$ are \emph{linearly independent}. The \emph{dimension} of a polyhedron $P$, denoted $\mathsf{dim}(P)$, is the cardinality of the largest affinely independent subset of $P$ \emph{minus one}. Given a polyhedron $P = \{\mathbf{x} \in \mathbb{R}^m \mid A\mathbf{x} \leq \bu\}$ a \emph{face} $F$ is another polyhedron $F = \{\mathbf{x} \in \mathbb{R}^n \mid A^\prime\mathbf{x} \leq \bu^\prime\}$, where $A^\prime$ is a submatrix of $A$ and $\bu^\prime$ is the respective subvector of $\bu$. A \emph{facet} is a face of $P$ with dimention $\mathsf{dim}(P) - 1$.

\subsection{Intersections of Cones}

\begin{figure}[h]
    \begin{minipage}{.5\textwidth}
        \begin{center}
            \includegraphics[scale=0.2]{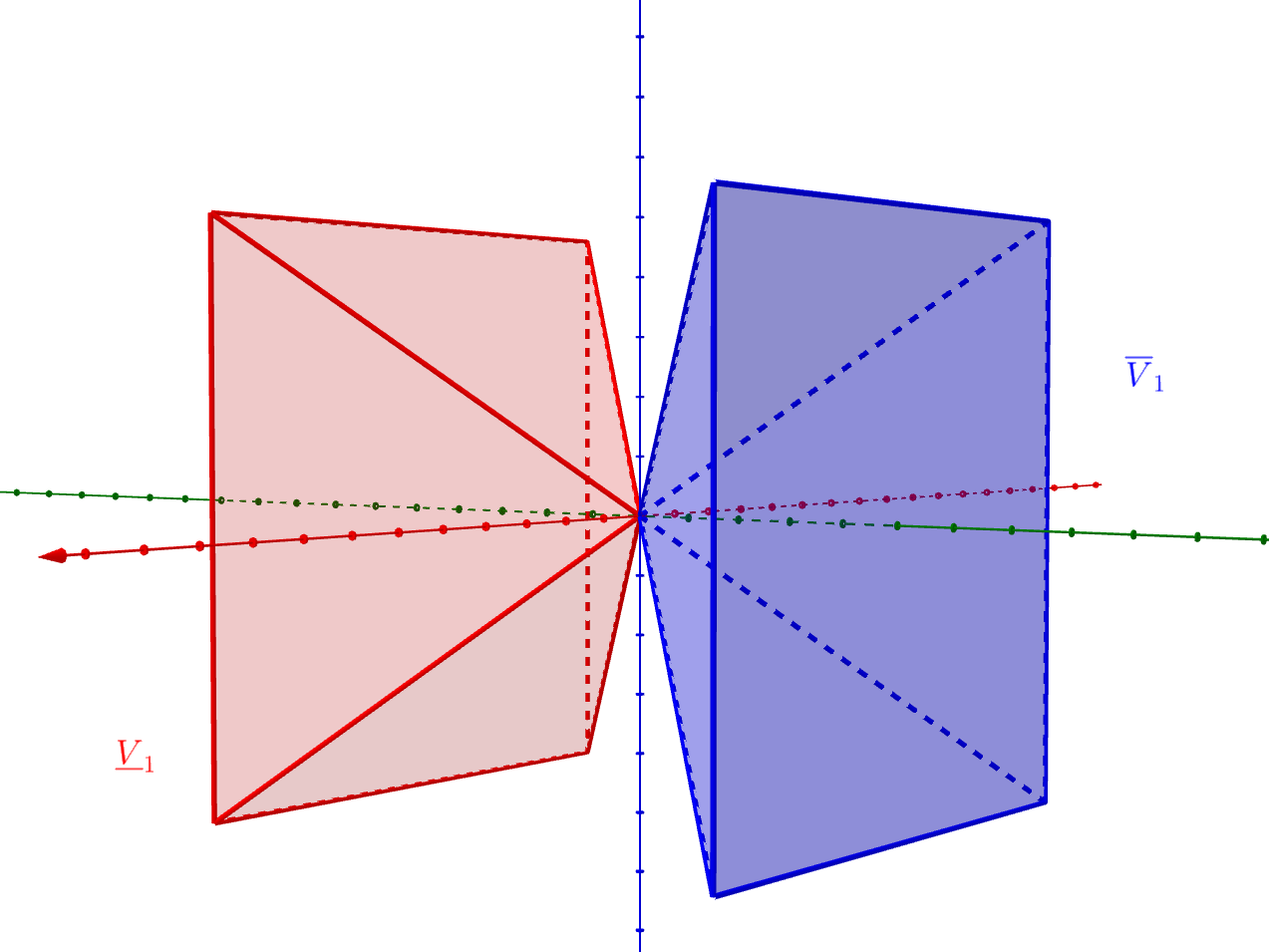}
        \end{center}
    \end{minipage}
    \begin{minipage}{.5\textwidth}
        \begin{center}
            \includegraphics[scale=0.2]{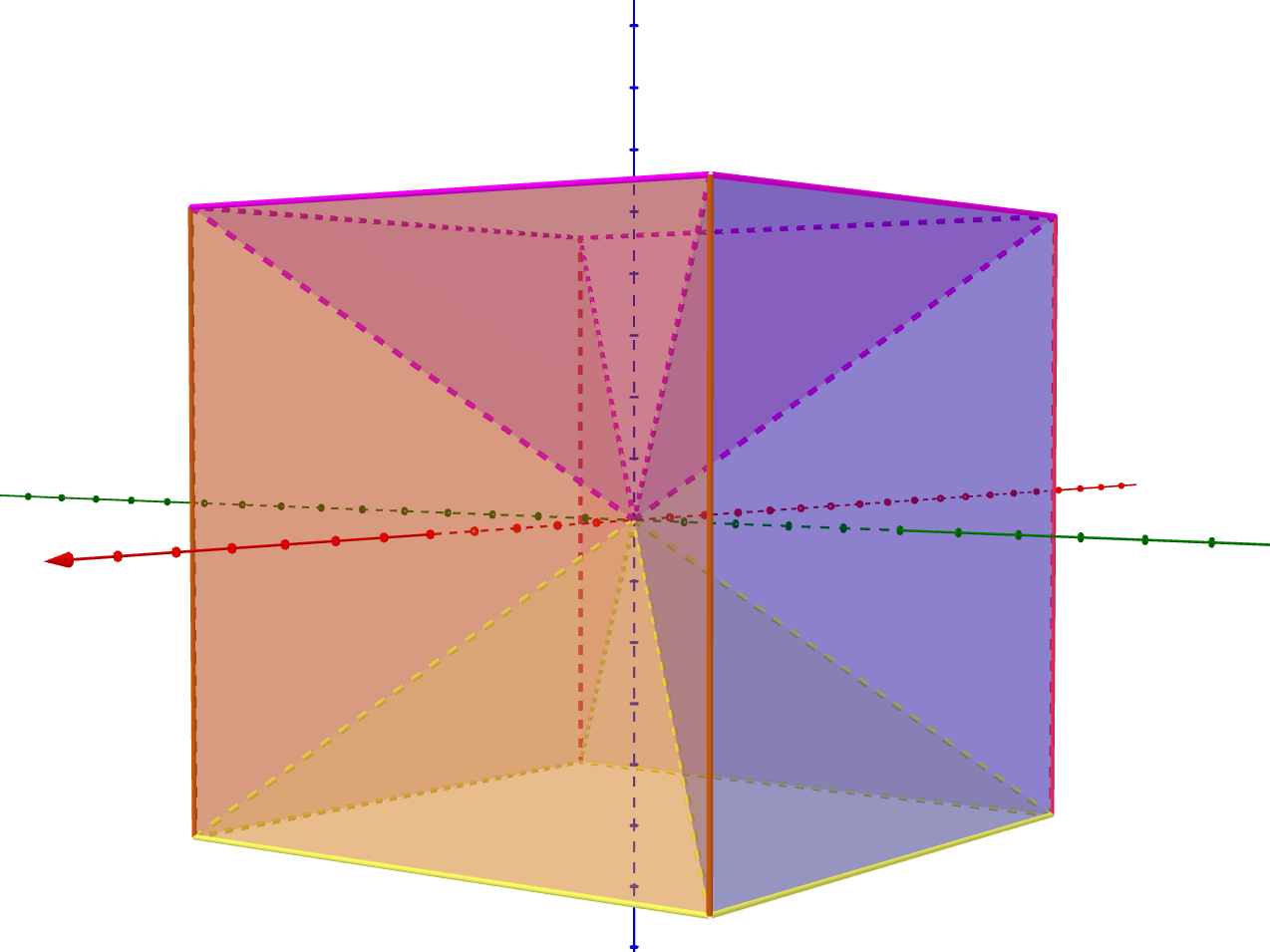}
        \end{center}
    \end{minipage}
    \caption{Cone Decomposition of $\mathbb{R}^3$. The polyhedral cones of eq. \eqref{eq:cones}, correspond to \emph{pyramids} is $\mathbb{R}^3$. Left: The $\underline{V}_1, \overline{V}_1$ cones of the $x_1$--axis. Observe that $\underline{V}_1 \cap \overline{V}_1 = \{\mathbf{0}\}$. Right: The cone decomposition ,
    of the entire $\mathbb{R}^3$.}
    \label{fig:cone-decomposition}
\end{figure}

To gain intuition about the underlying problem, we present the cone decomposition of $\mathbb{R}^3$ in Fig. \ref{fig:cone-decomposition}. Our focus is on the intersection of a collection of cones. We express the cones of eq. \eqref{eq:cones} as a set of inequalities. For each $i \in [d]$, let $\underline{U}_i, \overline{U}_i \in \mathbb{R}^{d \times d}$ denote two associated matrices. For the matrix $\underline{U}_i$, every element of the $i$-th column is $-1$. For all $m \neq n$ the entry, we have $(\underline{u}_i)_{mn} = 0$. Finally, for $n \neq i$, the diagonal entries satisfy $(\underline{u}_i)_{nn} = -1$. For $\overline{U}_i$, the construction is identical, except that every element in the $i$-th column equals $+1$. All other entries coincide with those of $\underline{U}_i$. In Fig. \ref{fig:cones-2}, we give an example when $d = 4$. With this notation in place, the cones from eq.~\eqref{eq:cones} can now be written equivalently in the form,
\begin{equation}
    \label{eq:cones-2}
    \underline{V}_i = \{\mathbf{x} \in \mathbb{R}^d \mid \underline{U}_i\mathbf{x} \geq \mathbf{0} \}, \quad\overline{V}_i = \{\mathbf{x} \in \mathbb{R}^d \mid \overline{U}_i\mathbf{x} \geq \mathbf{0} \}.
\end{equation}

\begin{figure}[t]
    \centering
    \begin{minipage}{.3\textwidth}
        \begin{equation*}
            \underline{U}_2 =
            \left[
            \begin{array}{c c c c }
                 -1 & -1 & 0 & 0\\
                 0 & -1 & 0 & 0\\
                 0 & -1 & -1 & 0\\
                 0 & -1 & 0 & -1\\
            \end{array}
            \right]
        \end{equation*}
    \end{minipage}
    \begin{minipage}{.3\textwidth}
        \begin{equation*}
            \overline{U}_3 =
            \left[
            \begin{array}{c c c c }
                 -1 & 0 & 1 & 0\\
                 0 & -1 & 1 & 0\\
                 0 & 0 & 1 & 0\\
                 0 & 0 & 1 & -1\\
            \end{array}
            \right]
        \end{equation*}
    \end{minipage}
    \begin{minipage}{.3\textwidth}
        \begin{equation*}
            \underline{\overline{U}}_{23} =
            \left[
            \begin{array}{c c c c }
                 -1 & -1 & 1 & 0\\
                 0 & -1 & 1 & 0\\
                 0 & -1 & 1 & 0\\
                 0 & -1 & 1 & -1\\
            \end{array}
            \right]
        \end{equation*}
    \end{minipage}
    \caption{Left, Middle: The matrices $\underline{U}_2,\overline{U}_3 \in \mathbb{R}^{4\times 4}$, for the cones $\underline{V}_2, \overline{V}_2$. It holds that $\underline{V}_2 \cap \overline{V}_3 = \{\mathbf{x} \in \mathbb{R}^d \mid \underline{\overline{U}}_{23}\mathbf{x} \geq \mathbf{0}\}$.}
    \label{fig:cones-2}
\end{figure}

Eq. \eqref{eq:cones-2} is equivalent to eq. \eqref{eq:cones}. Note that for each $i \in [d]$, the cones $\underline{V}_i$ and $\overline{V}_i$ both have full dimension, that is  $\mathsf{dim}(\underline{V}_i) = \mathsf{dim}(\overline{V}_i) = d$. Indeed, the rows of $\underline{U}_i$ form a set of $d$ linearly independent vectors. Together with the origin, these yield $d+1$ affinely independent vectors, confirming that $\underline{V}_i$ has dimension $d$. An identical argument applies to $\overline{V}_i$. 

Next, consider the intersection $\overline{V}_i \cap \overline{V}_j$ for distinct indices $i \neq j$. For any $\mathbf{x} \in \overline{V}_i \cap \overline{V}_j$ it must holds $\overline{U}_i\mathbf{x} \geq \mathbf{0}$ \emph{and} $\overline{U}_j\mathbf{x} \geq \mathbf{0}$. Thusly, we can define a matrix $\overline{U}_{ij}$, s.t. holds $\overline{U}_{ij}\mathbf{x} \geq \mathbf{0}$, \emph{iff} $\overline{U}_i\mathbf{x} \geq \mathbf{0}$ \emph{and} $\overline{U}_j\mathbf{x} \geq \mathbf{0}$. Concretely, $\overline{U}_{ij}$ is obtained from $\overline{U}_i$ by replacing the $j$-th column with all ones, while leaving all other entries unchanged. Hence,  $\overline{V}_i \cap \overline{V}_j = \{\mathbf{x} \in \mathbb{R}^d \mid \overline{U}_{ij} \mathbf{x} \geq \mathbf{0}\}$. An example for $d=4$ is shown in Fig. \ref{fig:cones-2}. Geometrically, this intersection reduces the dimension by one, $\mathsf{dim}(\overline{V}_i\cap \overline{V}_j) = d - 1$. Analogous constructions hold for the intersections $\underline{V}_i\cap \overline{V}_j$ and $\underline{V}_i\cap \underline{V}_j$. This agrees with the geometric intuition from Fig. \ref{fig:cone-decomposition}, since the intersection of two different pyramids in $\mathbb{R}^3$ is a facet.

    \section{Implementing the Maximal Closure Operator}
    \label{app:dichotomic}

In this section, we examine a \emph{deterministic} expansion method for computing maximally sound interval certifications. Essentially, the \textsc{Non--deterministic Expansion} method, introduced in Sec. \ref{sec:approximations}, and described in Algorithm \ref{algo:non-det-expand}, defines a \emph{collection} of deterministic algorithms, each making different choices on the coordinate to expand next. In Algorithm \ref{algo:expand-dichotomic}, we describe a simple, novel\footnote{
    The authors in \cite{iaxp} mention a dichotomic expansion as an alternative to their linear expansion algorithm. However, they do not give any details.
} and deterministic method, that expands each coordinate sequentially. At each step, we make a dichotomic (or binary) search to expand the $i$-th coordinate. Below, we establish the correctness and computational complexity of the above algorithm.

\begin{algorithm}[t]
    \DontPrintSemicolon
    \caption{\textsc{Non-DeterministicExpansion}}
    \label{algo:non-det-expand}
    \KwInput{$I_0 \in \intervalsdxu$, an initial sound interval certification, and $\delta > 0$, a percision constant.}
    \KwOutput{$J \in \intervalsdxu$, a expanded interval s.t. $J \supseteq I$.}

    $I \gets I_0$; $\underline{\mathcal{I}} \gets \varnothing$; $\overline{\mathcal{I}} \gets \varnothing$ 
    
    \While{$\underline{\mathcal{I}} \neq [d] \lor \overline{\mathcal{I}} \neq [d]$}{
        \textbf{choose} $\mathtt{endpoint} \in
  \{\mathcal{I} \in \{\underline{\mathcal{I}}, \overline{\mathcal{I}}\} \mid
  \mathcal{I} \neq [d]\}$
  
    \If{$\mathtt{endpoint} = \underline{\mathcal{I}}$}{
        \textbf{choose} $k \in [d] \setminus \underline{\mathcal{I}}$
        $J \gets I - [\delta\mathbf{e}_k, \mathbf{0}]$
        
            \If{$\models \mathscr{S}_{J, \mathcal{N}}$}{
                $\underline{\mathcal{I}} \gets \underline{\mathcal{I}} \cup \{k\}$
                
                \textbf{continue} \tcp*{Stop expanding this coordinate}
            }
    }\Else{
        \textbf{choose} $k \in [d] \setminus \overline{\mathcal{I}}$
        $J \gets I + [\mathbf{0}, \delta\mathbf{e}_k]$
        
        \If{$\models \mathscr{S}_{J, \mathcal{N}}$}{
                $\overline{\mathcal{I}} \gets \overline{\mathcal{I}} \cup \{k\}$
                
                \textbf{continue}
                \tcp*{Stop expanding this coordinate}
        }
    }
    
    $I \gets J$ \tcp*{Keep the expansion only if it is sound}
    }
\end{algorithm}

\begin{algorithm}[t]
    \DontPrintSemicolon
    \caption{\textsc{Sequential Dichotomic Expansion}}
    \label{algo:expand-dichotomic}
    \KwInput{an interval $[-\bell, \bu]$, $\bell, \bu \geq \mathbf{0}$, s.t. $[-\bell, \bu] \subseteq \mathcal{D}$.}
    \KwOutput{a maximally sound interval $[-\bell, \bu]$.}

    \For{$i \in [d]$}{
        $\mathtt{low}, \mathtt{high} \gets 0, \underline{U}_i$ \label{algo-step:expand-dichotomic-ub-while}
        \tcp*{expand $i$-th lower-bound}

        \While{$\mathtt{high} - \mathtt{low} > \delta$}{
            $\ell_i \gets \mathtt{low} +  (\mathtt{high} - \mathtt{low}) / 2$

            \If{$\mathscr{S}_{[-\bell, \bu], \mathcal{N}}(\mathbf{x})$}{
                $\mathtt{low} \gets \ell_i$
            }
            \Else{
                $\mathtt{high} \gets \ell_i$
            }
        }
        
        $\ell_i \gets \mathtt{low}$

        $\mathtt{low}, \mathtt{high} \gets 0, \overline{U}_i$ 
            \tcp*{expand $i$-th upper-bound}

        \While{$\mathtt{high} - \mathtt{low} > \delta$}{
            $u_i \gets \mathtt{low} + (\mathtt{high} - \mathtt{low}) / 2$

            \If{$\mathscr{S}_{[-\bell, \bu], \mathcal{N}}(\mathbf{x})$}{
                $\mathtt{low} \gets u_i$
            }
            \Else{
                $\mathtt{high} \gets u_i$
            }
        }            

        $u_i \gets \mathtt{low}$
    }

    \Return{$[-\bell, \bu]$}
\end{algorithm}

\begin{proposition}
    \label{prop:dichotomic}
    The \textsc{Sequential Dichotomic Expansion} algorithm satisfies Lemma \ref{lem:maximal-soundness}. Moreover, the algorithm terminates after $O(d\cdot\log(\mathcal{A}(\universe))$ $\mathscr{S}_{I, \mathcal{N}}(\cdot)$ oracle calls.
\end{proposition}
\begin{proof}
    Let $[-\bell, \bu]$ be the interval returned by the \textsc{Sequential Dichotomic Expansion} algorithm. We assume that $[-\bell, \bu] \subset \universe$. We prove that Lemma \ref{lem:maximal-soundness} is satisfied for the $i$--th coordinate of $\bu$. The case for the other coordinates of $\bu$ and the coordinates of $\bell$ are symmetrical.

    Consider the interval $[\mathtt{low}, \mathtt{high}]$ of step \ref{algo-step:expand-dichotomic-ub-while}. We show that if there is a counterexample in $(0, \underline{U}_i]$, there is always a counterexample in $(\mathtt{low}, \mathtt{high}]$, throughout the execution of the algorithm. Let $\mathtt{low}_j$, $u_{ij}$, $\mathtt{high}_j$ denote the values of the variables, of the $j$--th iteration. We prove this fact using induction.

    For the base step, $[\mathtt{low}_0, \mathtt{high}_0] = [0, \overline{U}_i]$. From our assumption, there is a counterexample $\mathbf{x}^a \in [\mathtt{low}_0, \mathtt{high}_0]$. Now assume that there is a counterexample in $[\mathtt{low}_j, \mathtt{high}_j]$. We show that there is a counterexample in $[\mathtt{low}_{j+1}, \mathtt{high}_{j+1}]$. Consider the variable $u_{i(j+1)}$ at the $(j+1)$--th iteration of the algorithm. It holds $u_{i(j+1)} = \mathtt{low}_j + (\mathtt{high}_j - \mathtt{low}_j)/2$. We take two cases, either there is a counterexample in $[\mathtt{low}_i, u_{i(j+1)}]$, or not.

    If there is a counterexample in $[\mathtt{low}_j, u_{i(j+1)}]$, then $\mathtt{low}_{j+1} = \mathtt{low}_j$. Moreover, $\mathtt{high}_{j+1} = u_{i(j+1)}$. Hence, exists a counter example in  $[\mathtt{low}_{j+1}, \mathtt{high}_{j+1}]$. Thus, we showed the desideratum.

    Now, we consider the case that there is no counterexample in $[\mathtt{low}_j, u_{i(j+1)}]$. But, from the inductive hypothesis, there is a counterexample in $[\mathtt{low}_j, \mathtt{high}_j]$. Therefore, there must be a counterexample in $[u_{i(j+1)}, \mathtt{high}_{j}]$. In this case, we have $\mathtt{low}_{j+1} = u_{i(j+1)}$ and $\mathtt{high}_{j+1} = \mathtt{high}_j$. Thus, there still exists a counterexample in $[\mathtt{low}_{j+1}, \mathtt{high}_{j+1}]$.

    From the step \ref{algo-step:expand-dichotomic-ub-while} of the algorithm, the procedure terminates when $\mathtt{high} -\mathtt{low} \leq \delta$. Moreover, $\ell_i = \mathtt{low}$. From the above, there is a counterexample in $[\mathtt{low}, \mathtt{high}]$. Thus, there is a counterexample in $[\bell, \bu + \delta\mathbf{e}^i]$.

    Finally, each expansion operation will take at most $\log(\mathcal{A}(\universe))$ steps. We make $d$ expansions. Moreover, the maximality of the returned solution is established by Lemma \ref{lem:maximal-soundness}. $\hfill \Box$
\end{proof}

    \section{More on Experimental Evaluation}
    \label{app:experiments}
    In this appendix we review some additional statistics from the experiments presented in Sec. \ref{sec:implementation}, 
providing additional insights on the details of algorithms.

\subsection{MNIST}

\begin{figure}[t]
    \centering
    \includegraphics[width=1\linewidth]{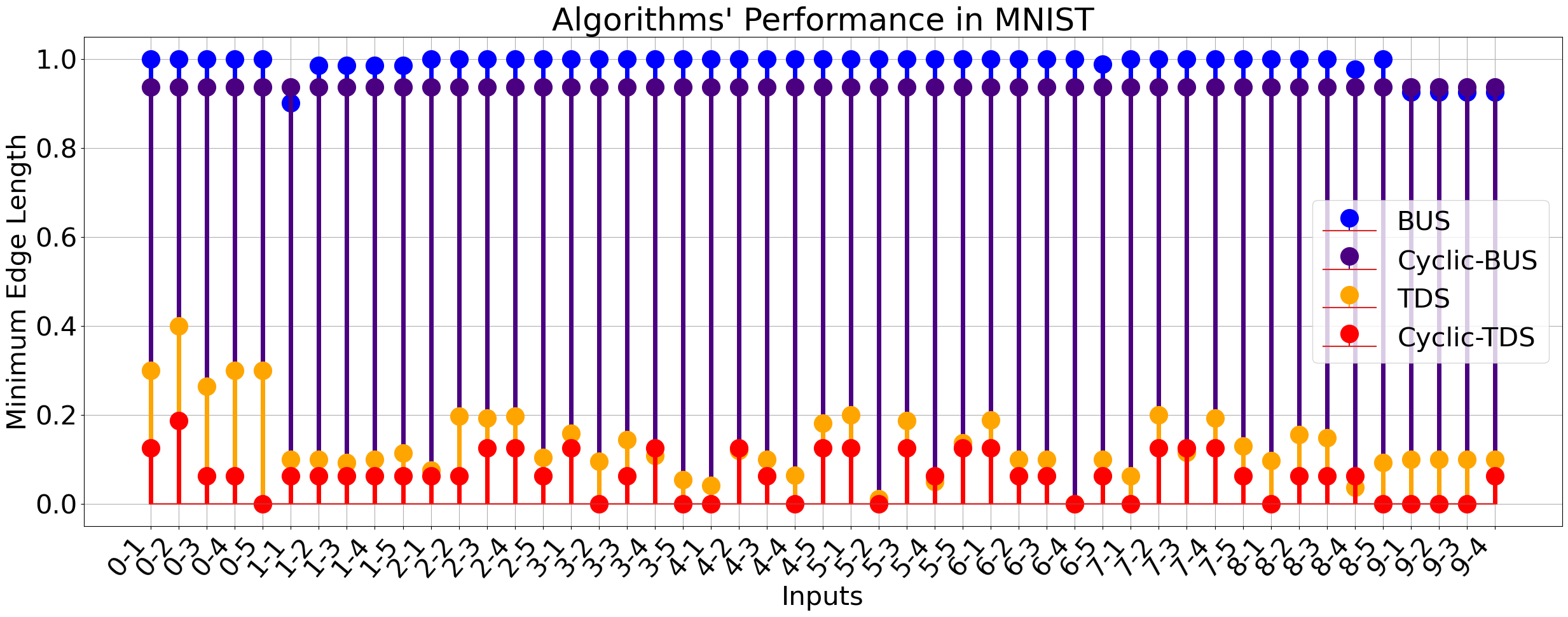}
    \caption{Minimum edge length of the certifications computed by each algorithm, per MNIST input point.}
    \label{fig:mnist-analytics}
\end{figure}

We begin from the MNIST dataset and the corresponding MLP. Tab. \ref{tab:experiments-mnist-expanded} provides a detailed description of the metrics analyzed in Sec. \ref{sec:implementation}. For each metric, we report the \emph{minimum (Min.)}, the \emph{average (Avg.)}, and the \emph{maximum (Max.)} values. In addition, we report the percentage of the inputs that timed out; the percentage of time consumed by the verification oracles; and the percentage of the non-trivial solutions, returned by each algorithm. Fig. \ref{fig:mnist-analytics} shows the achieved \emph{minimum edge length} for each input and algorithm. Finally, in Tab. \ref{tab:examples-MNIST} presents example images illustrating the computed bounds of the instance \textsc{7-4}.


\begin{table}
    \centering
    \begin{adjustbox}{width=\textwidth}
    \begin{tabular}{l c c c c c c}
                     & \textsc{BUS} & \textsc{TDS} & \textsc{SDE} & \textsc{TDS+SDE} & $\mathbb{B}$-\textsc{BUS} & $\mathbb{B}$-\textsc{TDS} \\   
         \thead{Lower\\ Bound}
            & ~\raisebox{-.5\height}{\includegraphics[scale=2]{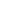}}~
            & ~\raisebox{-.5\height}{\includegraphics[scale=2]{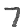}}~
            & ~\raisebox{-.5\height}{\includegraphics[scale=2]{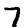}}~
            & ~\raisebox{-.5\height}{\includegraphics[scale=2]{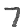}}~
            & ~\raisebox{-.5\height}{\includegraphics[scale=2]{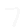}}~
            & ~\raisebox{-.5\height}{\includegraphics[scale=2]{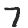}}~
            \\
         \thead{Upper\\ Bound}
            & \raisebox{-.5\height}{\includegraphics[scale=2]{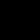}}
            & \raisebox{-.5\height}{\includegraphics[scale=2]{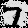}}
            & \raisebox{-.5\height}{\includegraphics[scale=2]{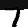}}
            & \raisebox{-.5\height}{\includegraphics[scale=2]{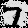}}
            & \raisebox{-.5\height}{\includegraphics[scale=2]{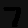}}
            & \raisebox{-.5\height}{\includegraphics[scale=2]{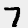}}
            \\
    \end{tabular}
    \end{adjustbox}
    \caption{Examples of the execution of all the presented algorithms on MNIST.}
    \label{tab:examples-MNIST}
\end{table}

From Tab. \ref{tab:examples-MNIST}, we observe that the \emph{morphology} of each bound strongly depends on the applied algorithm. Algorithms based on \emph{symmetric} interval (i.e, $\ell_\infty$-spheres) induce a uniform distortion on the given input. In contrast, sequential algorithms, such as \textsc{SDE}, expand each pixel in turn, once further expansion of the previous pixel is no longer possible. This results in a highly imbalanced distortion of the input image. 
Finally, the \textsc{TDS} and \textsc{BUS} algorithms lay between these two extremes.

\begin{table}
    \centering
    \begin{adjustbox}{width=\textwidth}
    \begin{tabular}{l l c c c c c c}
        \toprule\toprule
         \multirow{7}{*}{\thead{CPU\\Time}} & Algorithm & \thead{Min.\\ Value} & Avg. & \thead{Max.\\ Value} & Std Dev. & Timeouts & \thead{Timeout\\ Perc. (\%)} \\ \cline{2-8}
        & \textsc{BUS} & $30.29$m & $38.54$m &$49.70$m & $5.78$m & $0$ & 0\% \\
        &\textsc{TDS} & $8.97$m & $51.13$m &$62.54$m & $14.64$m & $28$ & 56\%\\ 
        &\textsc{SDE} & $11.27$m & $32.78$m & $67.95$m & $18.02$m & $13$ & 26\%\\
        &TDS+SDE & $+4.04$m & $+42.04$m & $+65.72$m & $14.64$m & 25 & 50\%\\
        \cline{2-8}
        &$\mathbb{B}$--BUS & $2.82$s & $3.92$s & $4.83$s & $0.55$s & $0$ & 0\% \\
        &$\mathbb{B}$--TDS & $1.05$s & $21.06$s & $149.26$s & $27.27$s & $0$ & 0\% \\
        \midrule\midrule
        \multirow{6}{*}{\thead{Number\\of\\Oracle\\Calls}} 
        &\textsc{BUS} & 1794 & 2085.73 & 2529 & 172.39 & 1.06s & 96\%\\
        &\textsc{TDS} & 1711 & 3753.02 & 4862 & 708.15 & 0.78s & 95\%\\
        &\textsc{SDE} & 1272 & 2963.86 & 3261 & 55.28 & 0.64s & 95\%\\
        &\textsc{TDS}+\textsc{SDE} & +10 & +1030.55 & +2569 & +927.37 & 2.41s & 98\%\\ 
        \cline{2-8}
        &$\mathbb{B}$--BUS & 4 & 4 & 4 & 0 & 0.94s & 99\%\\
        &$\mathbb{B}$--TDS & 4 & 4 & 4 & 0 & 0.93s & 95\%\\
        \midrule\midrule
        \multirow{6}{*}{\thead{Min.\\Edge\\Length\\$\alpha(\cdot)$}} 
        &\textsc{BUS} & 0.9 & 0.99 & 1.0 & 0.02 & 11 & 22\%\\
        &\textsc{TDS} & 0 & 0.13 & 0.4 & 0.08 & 48 & 96\%\\
        &\textsc{SDE}  & 0.0 & 0.0 & 0.0 & 0.0 & 0 & 0\%\\
        &\textsc{TDS}+\textsc{SDE} & 0.0 & 0.1 & 0.4 & 0.1 & 48 & 96\% \\
        \cline{2-8}
        &$\mathbb{B}$--BUS & 0.94 & 0.94 & 0.94 & 0.0 & 49 & 98\%\\
        &$\mathbb{B}$--TDS & 0.0 & 0.07 & 0.19 & 0.05 & 36 & 72\%\\
        \bottomrule\bottomrule
    \end{tabular}
    \end{adjustbox}
    \caption{Additional Analytics for the MNIST dataset.}
    \label{tab:experiments-mnist-expanded}
    \vspace{-1cm} 
\end{table}

\subsection{Fashion MNIST}

\begin{figure}
    \centering
    \includegraphics[width=1\linewidth]{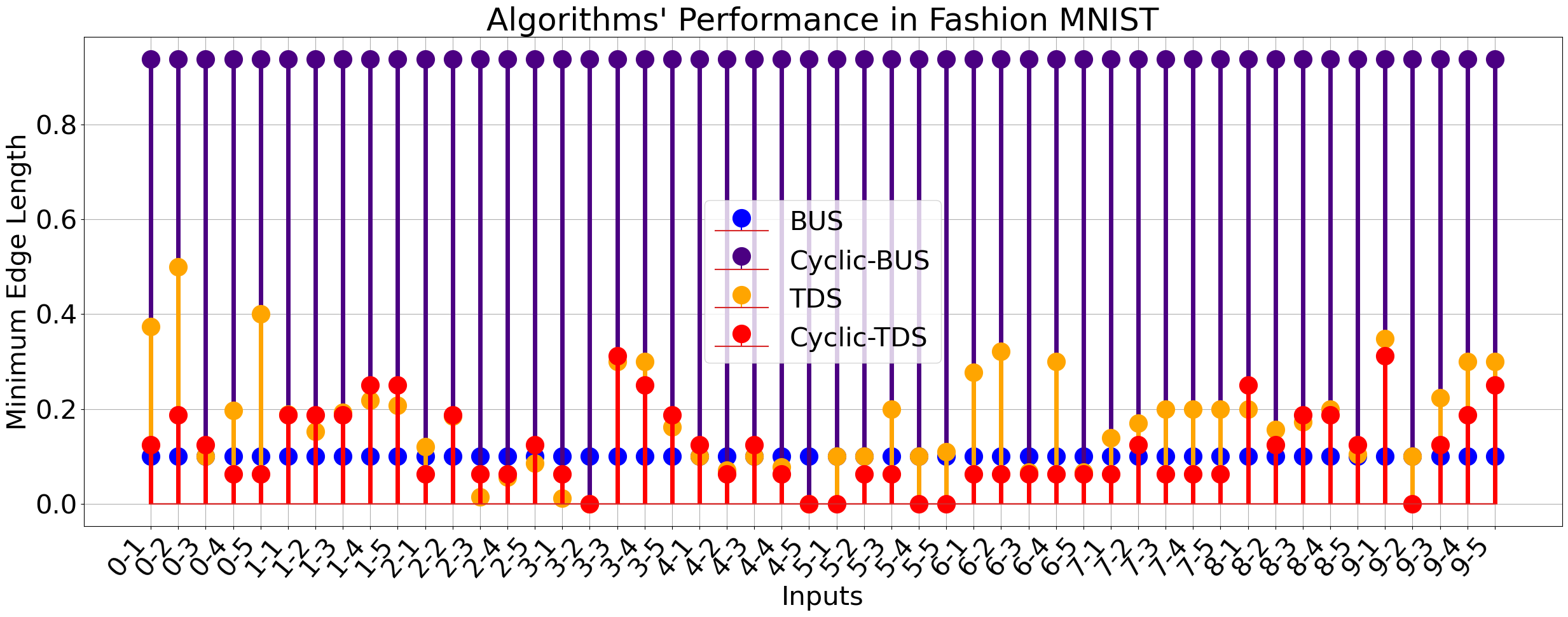}
    \caption{Minimum edge length of the certifications computed by each algorithm, per Fashion MNIST input point.}
    \label{fig:fashion-mnist-analytics}
\end{figure}

Subsequently, we focus on the Fashion MNIST dataset and the corresponding MLP. Tab. \ref{tab:experiments-fashion-mnist-expanded} provides a detailed description of the metrics used, reporting the \emph{minimum}, the \emph{average}, \emph{maximum)} values, along with its \emph{standard deviation}, similar to Tab. \ref{tab:experiments-mnist-expanded}. In addition, we report several percentage-based metrics. Fig. \ref{fig:fashion-mnist-analytics} reports the achieved \emph{minimum edge length} obtained for each input and algorithm. Finally, Tab. \ref{tab:examples-Fashion-MNIST} presents example images illustrating the computed bounds of the instance \textsc{7-4}.

From Fig. \ref{fig:fashion-mnist-analytics}, we observe that the \textsc{BUS} algorithm \emph{fails} to compute a correct complete certification, whereas $\mathbb{B}$-\textsc{BUS} succeeds. We attribute this behavior to \textsc{BUS}'s sensitivity to the precision parameter $\delta$. As shown in Tab. \ref{tab:examples-Fashion-MNIST}, the \textsc{BUS} algorithm behaves well; the small value of $\alpha$ is due to the involvement of only a few pixels.


\begin{table}
    \centering
    \begin{adjustbox}{width=\textwidth}
    \begin{tabular}{l c c c c c c}
                     & BUS & TDS & SDE & TDS+SDE & $\mathbb{B}$-BUS & $\mathbb{B}$-TDS \\   
         \thead{Lower\\ Bound}
            & ~\raisebox{-.5\height}{\includegraphics[scale=2]{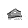}}~
            & ~\raisebox{-.5\height}{\includegraphics[scale=2]{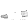}}~
            & ~\raisebox{-.5\height}{\includegraphics[scale=2]{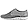}}~
            & ~\raisebox{-.5\height}{\includegraphics[scale=2]{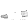}}~
            & ~\raisebox{-.5\height}{\includegraphics[scale=2]{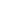}}~
            & ~\raisebox{-.5\height}{\includegraphics[scale=2]{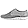}}~
            \\
         \thead{Upper\\ Bound}
            & \raisebox{-.5\height}{\includegraphics[scale=2]{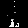}}
            & \raisebox{-.5\height}{\includegraphics[scale=2]{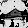}}
            & \raisebox{-.5\height}{\includegraphics[scale=2]{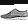}}
            & \raisebox{-.5\height}{\includegraphics[scale=2]{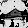}}
            & \raisebox{-.5\height}{\includegraphics[scale=2]{figures/examples-Fashion-MNIST/c-bu-d-7_4_lb.png}}
            & \raisebox{-.5\height}{\includegraphics[scale=2]{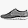}}
            \\
    \end{tabular}
    \end{adjustbox}
    \caption{Examples on Fashion MNIST.}
    \label{tab:examples-Fashion-MNIST}
\end{table}

\begin{table}
    \centering
    \begin{adjustbox}{width=\textwidth}
    \begin{tabular}{l l c c c c c c}
        \toprule\toprule
        \multirow{7}{*}{\thead{CPU\\Time}} & Algorithm & \thead{Min.\\ Value} & Avg. & \thead{Max.\\ Value} & Std Dev. & Timeouts & \thead{Timeout\\ Perc. (\%)} \\ \cline{2-8}
        & \textsc{BUS} & 17.1m & 23.82m & 30.87m & 3.45m & 0 & 0\%\\
        & \textsc{TDS} & 13.48m & 50.17m & 61.15m & 15.31m & 32 & 64\%\\ 
        & \textsc{SDE} & 10.38m & 22.82m & 60.03m & 13.27m & 3 & 6\%\\
        & TDS+SDE & +6.52m & +23.91m & +60.01m & 13m & 1 & 2\%\\
         \cline{2-8}
        &$\mathbb{B}$--BUS & 2.42s & 3.52s & 4.47s & 0.63s & 0 & 0\%\\
        &$\mathbb{B}$--TDS & 0.76s & 5.15s & 28.48s & 6.1s & 0 & 0\%\\
        \midrule\midrule
        \multirow{6}{*}{\thead{Number\\of\\Oracle\\Calls}} 
        &\textsc{BUS} & 987.12 & 1375.41 & 1782.78 & 200.99 & 1.11s & 96\%\\
        &\textsc{TDS} & 1343 & 3272.76 & 5062 & 724.22 & 0.88s & 95\%\\
        &\textsc{SDE} & 1450 & 3469.48 & 4040 & 390.51 & 0.36s & 92\%\\
        &\textsc{TDS}+\textsc{SDE} & +819 & +1779.34 & +2757 & 386.36 & 0.77s & 96\%\\ 
        \cline{2-8}
        &$\mathbb{B}$--BUS & 4 & 4 & 4 & 0.0 & 0.85s & 96\%\\
        &$\mathbb{B}$--TDS & 4 & 4 & 4 & 0.0 & 1.26s & 98\%\\
        \midrule\midrule
        \multirow{6}{*}{\thead{Min.\\Edge\\Length\\$\alpha(\cdot)$}} 
        &\textsc{BUS} & 0.1 & 0.1 & 0.1 & 0.0 & 50 & 100\% \\
        &\textsc{TDS} & 0.0 & 0.18 & 0.5 & 0.11 & 48 & 96\%\\
        &\textsc{SDE} & 0.0 & 0.0 & 0.0 & 0.0 & 0 & 0\%\\
        &\textsc{TDS}+\textsc{SDE} & 0.0 & 0.18 & 0.5 & 0.11 & 48 & 96\%\\
        \cline{2-8}
        &$\mathbb{B}$--BUS & 0.94 & 0.94 & 0.94 & 0.0 & 50 & 100\%\\
        &$\mathbb{B}$--TDS & 0.0 & 0.12 & 0.31 & 0.08 & 44 & 88\%\\
        \bottomrule\bottomrule
    \end{tabular}
    \end{adjustbox}
    \caption{Additional Experimental Statistics for the Fashion MNIST dataset.}
    \label{tab:experiments-fashion-mnist-expanded}
\end{table}

\end{document}